\documentclass{article}

\PassOptionsToPackage{numbers, compress}{natbib}

\usepackage[preprint]{neurips_2021}

\usepackage[utf8]{inputenc} 
\usepackage[T1]{fontenc}    
\usepackage{url}            
\usepackage{booktabs}       
\usepackage{amsfonts}       
\usepackage{nicefrac}       
\usepackage{microtype}      
\usepackage{xcolor}         
\usepackage{bbold}
\usepackage{comment}

\usepackage{epsfig,amsmath,amssymb,amsfonts,amstext,amsthm,mathrsfs,mathtools}
\usepackage{latexsym,graphics,epsf,epsfig,psfrag}
\usepackage{dsfont,color,epstopdf,fixmath}
\usepackage{enumitem}

\usepackage[colorlinks=true,citecolor=blue]{hyperref}

\usepackage{hyperref}
\usepackage{cleveref}
\usepackage{url}

\usepackage{algorithm,algorithmic}

\newcommand{\mvec}{\text{vec}}
\newcommand{\msP}{\mathsf{P}}
\newcommand{\msI}{\mathsf{I}}

\newcommand{\mE}{\mathbb E}

\newcommand{\mcs}{\mathcal S}
\newcommand{\mca}{\mathcal A}

\newcommand{\mcf}{\mathcal F}
\newcommand{\mct}{\mathcal T}

\newcommand{\mcd}{\mathcal D}

\newcommand{\mf}{\mathcal F}

\newcommand{\rpi}{{\rm\Gamma}}
\newcommand{\rpii}{{\rm\Pi}}
\newcommand{\mcphi}{{\rm\Phi}}

\newcommand{\ltwo}[1]{\left\|#1\right\|_2}
\newcommand{\lsd}[1]{\left\|#1\right\|_D}

\newcommand{\lcin}[1]{\left\|#1\right\|_{C^{-1}}}
\newcommand{\lone}[1]{\left|#1\right|}

\newcommand{\lxi}[1]{\left\|#1\right\|_\xi}
\newcommand{\lF}[1]{\left\|#1\right\|_F}

\newcommand{\lalpha}[1]{\left\|#1\right\|_{\mu_\pi,\alpha}}

\newcommand{\lmu}[1]{\left\|#1\right\|_{U_\pi}}
\newcommand{\lsmu}[1]{\left\|#1\right\|_{\mu_\pi}}

\newcommand{\diag}{{\rm diag}}

\newcommand{\mR}{\mathbb{R}}

\newtheorem{theorem}{Theorem}

\newtheorem{example}{Example}

\newtheorem{lemma}{Lemma}
\newtheorem{proposition}{Proposition}
\newtheorem{assumption}{Assumption}

\newtheorem{definition}{Definition}

\allowdisplaybreaks[4]

\DeclareMathOperator*{\argmin}{argmin}
\DeclareMathOperator*{\argmax}{argmax}

\title{A Unified Off-Policy Evaluation Approach for General Value Function}

\author{%
  Tengyu Xu \\
  The Ohio State University\\
  \texttt{xu.3260@osu.edu} \\
   \And
   Zhuoran Yang \\
   Princeton University \\
   \texttt{zy6@princeton.edu} \\
   \AND
   Zhaoran Wang \\
   Northwestern University \\
   \texttt{zhaoranwang@gmail.com} \\
   \And
   Yingbin Liang \\
   The Ohio State University \\
   \texttt{liang.889@osu.edu} \\
}

\begin{document}

\maketitle

\begin{abstract}
	General Value Function (GVF) is a powerful tool to represent both the {\em predictive} and {\em retrospective} knowledge in reinforcement learning (RL). In practice, often multiple interrelated GVFs need to be evaluated jointly with pre-collected off-policy samples. In the literature, the gradient temporal difference (GTD) learning method has been adopted to evaluate GVFs in the off-policy setting, but such an approach may suffer from a large estimation error even if the function approximation class is sufficiently expressive. Moreover, none of the previous work have formally established the convergence guarantee to the ground truth GVFs under the function approximation settings. In this paper, we address both issues through the lens of a class of GVFs with causal filtering, which cover a wide range of RL applications such as reward variance, value gradient, cost in anomaly detection, stationary distribution gradient, etc. We propose a new algorithm called GenTD for off-policy GVFs evaluation and show that GenTD learns multiple interrelated multi-dimensional GVFs as efficiently as a single canonical scalar value function. We further show that unlike GTD, the learned GVFs by GenTD are guaranteed to converge to the ground truth GVFs as long as the function approximation power is sufficiently large. To our best knowledge, GenTD is the first off-policy GVF evaluation algorithm that has global optimality guarantee.

\end{abstract}

\section{Introduction}
The value function, which represents the expected accumulation of reward \cite{sutton2009grand}, serves as a reliable performance metric of policy in the reinforcement learning (RL) tasks \cite{sutton1988learning,maei2011gradient}, 
In many RL applications, however, looking at only the value function is not enough. For example, in the risk-sensitive domains such as health care and financial assets, the variance of "reward-to-go" rather than the value function, i.e., the mean of "reward-to-go", is a more suitable performance metric. As another example, to obtain a variance-reduced or bias-reduced policy gradient estimator \cite{huang2020importance,xu2021doubly,kallus2020statistically}, in addition to the value function, the information of "gradient of value function" is also required. Moreover, in continuous control domain with differentiable and deterministic policy, the computation of policy gradient is only possible through "action/state-value gradient" \cite{silver2014deterministic,d2020learn,heess2015learning}, etc. All the aforementioned metrics can be viewed as {\em predicative} knowledge of certain cumulative "signals" (possibly high-dimensional, e.g., the gradient of value function), and thus naturally fall into the framework of {\bf forward GVFs} (refers to forward general value functions) \cite{sutton2011horde,white2015developing,schaul2013better}. One typical approach to evaluate GVFs, is to learn from samples that pre-collected from one or more behavior policies, which yields an {\em off-policy} method. In practice, multiple forward GVFs are usually evaluated jointly at the same time due to their interrelationships \cite{sutton2011horde,silver2013gradient}. 


In contrast to forward GVFs defined based on predictive knowledge, the {\bf backward GVF} represents {\em retrospective} knowledge, which captures the accumulation of signals from the past to the present time \cite{zhang2020learning}. Although the concept of the backward GVF has not been formally proposed until very recently \cite{zhang2020learning}, it is rooted in a number of important RL applications such as anomaly detection \cite{zhang2020learning}, emphatic weight learning \cite{sutton2016emphatic,zhang2020learning} and evaluation of gradient of logarithmic stationary distribution \cite{morimura2010derivatives,xu2021doubly,kallus2020statistically}. Differently from the forward GVF, for which the Bellman operator can be defined independently from the sampling distribution \cite{silver2013gradient,sutton1988learning,sutton2018reinforcement}, the Bellman operator of the backward GVF is only valid if the sampling exactly follows the on-policy stationary distribution \cite{zhang2020learning}. Due to such a reason, off-policy evaluation of the backward GVF is much more challenging than that of the forward GVF.

In previous studies, the gradient temporal difference (GTD) learning \cite{sutton2009fast,maei2011gradient}, one of the most popular off-policy methods in value function evaluation, has been adopted to solve both the forward and backward GVF evaluation problems \cite{sutton2011horde,silver2013gradient,zhang2020provably}. GTD adopts the mean squared projected Bellman error (MSPBE) as its optimization objective and takes the expectation over the {\em behavior} policy, which does not exactly reflect the desirable evaluation under the {\em target} policy. As a result, GTD can encounter serious issues in GVF evaluation problems. {\bf First}, the optimal point to which GTD converges can be far away from the ground truth value of GVFs. It becomes worse when multiple GVFs are evaluated simultaneously, because the error of one GVF evaluation can be further amplified across other GVFs' evaluation due to their inherent correlations.
In the literature, no provable bound has been established on such an error, which can, in fact, be unbounded for some cases (see \cite[Example 1]{kolter2011fixed}). {\bf Second}, for high-dimensional GVFs evaluations, the landscape geometry of the GTD objective function can be ill-conditioned \cite{maei2011gradient}, which could slow down the convergence of GTD significantly. As demonstrated by our empirical results in \Cref{sc: exp}, GTD can suffer from both the large estimation error and the slow convergence rate, which further suggests that GTD may not be a good choice for GVFs evaluation tasks. This motivates our paper to address the following question:
%
\begin{list}{$\bullet$}{\topsep=0.ex \leftmargin=0.15in \rightmargin=0.in \itemsep =0.01in}
	\item {\em Can we design a new off-policy approach for multiple interrelated and high-dimensional GVFs evaluation problems, which is guaranteed to converge fast and converge to the ground truth GVFs?}
\end{list}
%
%

{\bf Our Contributions.} In this paper, we investigate the problem of evaluating multiple interrelated GVFs jointly. Rather than studying different GVFs on a case-by-case basis, we explore the class of "GVFs with causal filtering", which captures a common structural feature shared by GVFs in a wide range of RL applications (see \Cref{sc: application}). {\bf (a)} We prove that both forward and backward GVFs with causal filtering are the unique fixed point of their corresponding general Bellman operator (GBO) (defined for multiple high-dimensional GVFs), which is shown to have a contraction property with respect to a properly constructed norm metric. 
{\bf (b)} Based on such a property of GVFs, we propose a new algorithm GenTD to solve off-policy GVFs evaluation problem. GenTD introduces a density ratio to adjust the behavior distribution and further incorporates a policy-agnostic approach GenDICE/GradientDICE \cite{zhang2020gendice,zhang2020learning} for estimating the density ratio jointly with GVF evaluation. {\bf (c)} In the linear function approximation setting, we show that GenTD converges to the globally optimal point at the rate of $\mathcal{O}(1/T)$, with conditional number independent from the dimension of GVFs. Such a result implies that GenTD learns multiple interrelated possibly high-dimensional GVFs as efficiently as TD learning for a single canonical scalar value function. {\bf (d)} We further show that unlike GTD, GenTD are guaranteed to converge to the ground truth GVFs as long as the function expressive power is sufficiently large. To our best knowledge, GenTD is the first off-policy GVF evaluation algorithm that has gound truth guarantee. {\bf (e)} Our experiments further demonstrate that GenTD converges much faster than GTD, and more importantly, converges to ground truth, whereas GTD can stay far way from the ground truth GVF value.
{\bf Related Work.} The forward GVF was first introduced in \cite{sutton2011horde} to represent a set of accumulation of general signals with possibly time-varying discount factors. The forward GVF was later used to represent a set of interrelated predictions \cite{silver2013gradient,downey2017predictive,sun2016learning,mahmood2013representation}. It has been observed that some RL metrics such as variance, gradient of value function, state/action value gradient can also be viewed as forward GVFs \cite{tamar2016learning,huang2020importance,xu2021doubly,kallus2020statistically,silver2014deterministic,sutton2011horde,white2015developing,schaul2013better,d2020learn,comanici2018knowledge}. In previous works, both TD learning and GTD have been used to evaluate forward GVFs in the on- and off-policy settings \cite{sutton2011horde,silver2013gradient}, respectively. A more comprehensive review of studies of forward GVFs has been provided in \cite{sherstan2020gradient}.
The backward GVF was formally defined in \cite{zhang2020learning}. Some previous works have also considered metrics that can be represented as accumulations of signals in the reverse time direction, such as emphatic weighting, page ranking cost, and derivative of logarithmic stationary distribution \cite{zhang2019generalized,zhang2019provably,morimura2010derivatives,yao2013reinforcement,hallak2017consistent}. Another track of research has focused on evaluation of a general scalar function in the off-policy setting \cite{chandak2021universal,zhang2020variational,jiang2016doubly}, whereas the focus of this paper is on the evaluation of multiple high-dimensional GVFs.

The theoretical studies of off-policy GVFs evaluation algorithms are rather limited. So far, only the asymptotic convergence guarantee (without the convergence rate characterization) of GTD has been established in both the forward and backward GVFs evaluation settings \cite{silver2013gradient,zhang2020learning}. The convergence rate of GTD has only been established in \cite{xu2019two,dalal2018finite,kaledin2020finite,dalal2020tale,xu2021sample,liu2015finite} for the simple canonical value function evaluation setting, which is a special case of forward GVFs. However, as pointed out in \cite{kolter2011fixed,hallak2017consistent,munos2003error}, the optimal point of GTD may suffer from possibly unbounded approximation error, which is not desirable in practice. In contrast, we propose a new off-policy GVFs evaluation algorithm, which can solve a wide range of forward and backward GVFs evaluation problems, with convergence rate characterization and guaranteed optimality with respect to the ground truth GVF value.

\section{Markov Decision Process and General Value Function}\label{sc: background}
We consider an infinite-horizon Markov Decision Process (MDP) with a state space $\mcs$, an action space $\mca$, a reward function $r:\mcs\times\mca \rightarrow\mR$, a transition kernel $\msP: \mcs\times\mcs\times\mca\rightarrow [0,1]$, a discounted factor $\gamma\in(0,1)$, and an initial distribution $\mu_0:\mcs\rightarrow[0,1]$. An policy $\pi(a|s)$ is the probability of taking action $a$ at state $s$. 
At time step $t$, an agent at a state $s_t$ selects an action $a_t$ according to $\pi(\cdot|s_t)$, receives a reward $r(s_t,a_t)$, and transits to state $s_{t+1}$ according to $\msP(\cdot|s_t,a_t)$. The state-action transition kernel is defined as $\msP_\pi\in\mR^{|\mcs||\mca|\times |\mcs||\mca|}$, in which $\msP_\pi((s,a),(s^\prime,a^\prime))=\msP(s^\prime|s,a)\pi(a^\prime|s^\prime)$. When the MDP is ergodic, we define $\mu_\pi$ as the state-action stationary distribution which satisfies: $\mu_\pi^\top \msP_\pi = \mu_\pi^\top$.
For such an MDP, we define the discounted accumulation of reward as the "reward-to-go": $J_\pi = \sum_{t=0}^{\infty}\gamma^t r(s_t,a_t)$. The state-action value function (i.e., Q-function) is defined as $Q_\pi(s,a) = \mE[J_\pi|(s_0,a_0)=(s,a)]$, and the state value function (i.e., V-function) is defined as $V_\pi(s) = \mE[Q_\pi(s,a)|s]$. Note that $Q_\pi(s,a)$ satisfies the following Bellman equation
\begin{flalign}\label{normal_BE}
	Q_\pi = \mct_\pi Q_\pi =  R + \gamma \msP_\pi Q_\pi,
\end{flalign}
where $\mct_\pi$ is the Bellman operator, and $Q_\pi$, and $R\in \mR^{|\mcs||\mca|}$ are vectors obtained via stacking $Q_\pi(s,a)$ and $r(s,a)$ over state-action space $\mcs\times\mca$. We introduce a function of $(s,a)$ (possibly in the vector form) as $v(s,a)\in\mR^d$ ($d\geq 1$). Consider a distribution $\xi(\cdot)$ over $\mcs\times\mca$.
We define the $\xi$--norm of $v\in\mR^{d|\mcs||\mca|}$ as $\lxi{v}=\sqrt{\sum_{(s,a)} \xi(s,a)\ltwo{v(s,a)}^2}$, where $v$ is obtained by stacking the function $v(s,a)$ over $\mcs\times\mca$.
It has been proved that $\mct_\pi$ is $\gamma$--contraction in $\mu_\pi$--norm, i.e., $\lsmu{\mct_\pi v - \mct_\pi v^\prime}\leq \gamma \lsmu{v-v^\prime}$ and $Q_\pi$ is the unique fixed point of $\mct_\pi$ \cite{sutton2018reinforcement,sutton1988learning,tsitsiklis1997analysis}. 
In the sequel, we denote $\msI_d$ as the identity matrix with the dimension $d$ and $\otimes$ as the Kronecker product. We further define $U_\pi = \text{diag}(U_{\pi,1},\cdots,U_{\pi,k})$, in which $U_{\pi,i} = \text{diag}(\mu_{\pi})\otimes \msI_{d_i}$ for $i=\{1,\cdots,k \}$, and $P_\pi=\text{diag}(P_{\pi,1},\cdots,P_{\pi,k})$, in which $P_{\pi,i} = \msP_\pi\otimes \msI_{d_i}$.

\subsection{Forward General Value Function}\label{sc: fgvf}

Consider a set of the state-action general value functions (GVFs) $G_\pi  = [G_{\pi,1}^\top, \cdots, G_{\pi,k}^\top]^\top$, where each GVF $G_{\pi,i}$ is defined as the accumulation of a corresponding signal $C_i(s,a) \in \mR^{d_i}$
given by
\begin{flalign}\label{eq: gvf}
\textstyle	G_{\pi, i}(s,a)=\mE\left[\sum_{t=0}^{\infty}\gamma_i^t C_i(s_t,a_t)\big|(s_0,a_0)=(s,a),\pi\right],
\end{flalign}
where $\gamma_i\in(0,1)$ is a discount factor associated with $C_i$. Since $C_i(s,a) \in \mR^{d_i}$ (for each $(s,a)$) can be high-dimensional, $G_{\pi, i}(s,a)$ can also be high-dimensional for each $(s,a)$. Clearly, the Q-function is a special GVF associated with a scalar signal. Since $G_\pi$ is defined as the accumulation of the signal $C_i$ in a forward direction from the current time step $t$ to the future $\infty$, we call $G_\pi$ as "forward GVF". 

In many RL applications, GVFs share a commen structure of causal filtering \cite{sutton2011horde}, i.e., each $C_i(s,a)$ (associated with $G_{\pi,i}$) depends on the lower-indexed value functions $G_{\pi,1}, \cdots, G_{\pi,i-1}$ in the set. As a concrete example, suppose the policy is parametrized by a smooth function $\pi_w$, where the parameter $w\in\mR^{d_w}$. In addition to the Q-function $Q_\pi$, the gradient $\nabla_w Q_{\pi}(s,a)$ of the Q-function w.r.t. $w$ arises as a GVF of interest in several important applications such as variance reduced policy gradient \cite{huang2020importance} and on- and off-policy policy optimization \cite{xu2021doubly,silver2014deterministic,kallus2020statistically,comanici2018knowledge}. In such a case, let $G_{\pi,1}=Q_\pi$ and $G_{\pi,2}=\nabla_w Q_{\pi}$. Further, it has been shown in \cite{xu2021doubly,kallus2020statistically,comanici2018knowledge} that the signal $C_2(s,a)$ associated with $\nabla_wQ_\pi$ is given by $C_2(s,a)=\gamma \mE[Q_\pi(s^\prime,a^\prime)\nabla_w\log(\pi_w(s^\prime,a^\prime))|s,a]$, which depends on the lower-indexed $G_{\pi,1}=Q_\pi$. Hence, such a GVF vector has the causal filtering structure. \Cref{sc: application} provides further details about this example and more such GVF examples in RL.
More formally, we define the forward GVF with causal filtering as follows. 
\begin{definition}[Forward GVF with causal filtering]\label{def:forwardGVF}
For a given policy $\pi$, 
a forward GVF $G_\pi= [G_{\pi,1}^\top, \cdots, G_{\pi,k}^\top]^\top$ with causal filtering are associated with signals satisfying
\begin{align*}
\textstyle C_i=B_i+  \sum_{j=1}^{i-1}A_{i,j}G_{\pi,j} \quad \text{for }\; 2\leq i \leq k,
\end{align*}
where $C_i$ and $G_{\pi,j}$ are obtained by respectively stacking $C_i(s,a)\in\mR^{d_i}$ and $G_{\pi,j} (s,a)\in\mR^{d_i}$ over $\mcs\times\mca$, $B_i\in\mR^{d_i|\mcs||\mca|}$ is an observable signal, and the coefficient matrix $A_{i,j}\in \mR^{d_i|\mcs||\mca|\times d_j|\mcs||\mca|}$ captures how the $j$-th GVF $G_{\pi,j}$ affects the $i$-th accumulation signal $C_i$. Further, $B_i$ and $A_{i,j}$ are bounded for all $i,j=1,\cdots,k$ to ensure $G_{\pi,i}$ to be well defined.
\end{definition}

\Cref{def:forwardGVF} indicates that all GVFs are interrelated with a causal filtering structure, i.e., each signal $C_i$ is a linear function of all lower-indexed $G_{\pi,l}$ for $1\leq l< i$. {Such a structure also captures the core nature of the TD net \cite{sutton2004temporal}, in which the prediction of one node may depend on the outputs from some previous nodes.}
\Cref{def:forwardGVF} also implies that the forward GVF $G_\pi = [G^\top_{\pi,1}, \cdots, G^\top_{\pi,k}]^\top$ with causal filtering satisfies the following {\bf lower-triangular Bellman equation} given by
\begin{flalign}\label{forward_GBE}
	G_\pi = \mct_{G,\pi} G_\pi = B + M_\pi G_\pi,
\end{flalign}
where where $\mct_{G,\pi}$ denotes the forward general Bellman operator (GBO), and
\begin{flalign*}
	B = \left[\begin{array}{c}
	B_1\\
	B_2\\
	\vdots\\
	B_k
	\end{array}
	\right],\quad
	M_\pi = \left[\begin{array}{cccc}
	\gamma_1[\msP_\pi\otimes\msI_{d_1}] & 0 & \cdots  & 0 \\
	A_{2,1} & \gamma_2 [\msP_\pi\otimes\msI_{d_2}] & \cdots  & 0\\
	\vdots & \vdots & & \vdots \\
	A_{k,1} & A_{k,2} & \cdots & \gamma_k [\msP_\pi\otimes\msI_{d_k}] 
	\end{array}
	\right].
\end{flalign*}
Clearly, the canonical value function $Q_\pi$ and Bellman operator $\mct_\pi$ defined in \cref{normal_BE} is a special case of $G_\pi$ and $\mct_{G,\pi}$ defined in \cref{forward_GBE}. 

\subsection{Backward General Value Function}\label{sc: bgvf}
In contrast to the forward GVF defined in the last section, which represents the {\em predictive knowledge}, in some RL scenarios, we also want to capture the {\em retrospective knowledge} (see \Cref{sc: application} for concrete examples), which represents the accumulation of signals that have been collected from the past. Consider a set of GVFs $\hat{G}_\pi  = [\hat{G}_{\pi,1}^\top, \cdots, \hat{G}_{\pi,k}^\top]^\top$, where each GVF $\hat{G}_{\pi,i}$ is defined as the {\em backward} accumulation of a vector signal $\hat{C}_i(s,a) \in \mR^{d_i}$ given by
\begin{flalign}\label{eq: backgvf}
\textstyle \hat{G}_{\pi, i}(s,a)=\mE\big[\sum_{t=-\infty}^{0}\gamma_i^{-t} \hat{C}_i(s_t,a_t)\big|(s_0,a_0)=(s,a),\pi\big].
\end{flalign}
To distinguish from the forward GVF $G_{\pi,i}$ defined in \cref{eq: gvf}, we denote $\hat{G}_{\pi,i}$ as the backward GVF. For general purpose, we also consider the causal filtering setting for $\hat{G}_\pi$, in which each $\hat{C}_i(s,a)$ depends on the lower-indexed value functions $\hat{G}_{\pi,1}, \cdots, \hat{G}_{\pi,i-1}$ in the set. We define the backward GVF with causal filtering as follows.

\begin{definition}[Backward GVF with causal filtering]\label{def:backwardGVF}
	For a given policy $\pi$, 
	a backward GVF $\hat{G}_\pi= [\hat{G}_{\pi,1}, \cdots, \hat{G}_{\pi,k}]$ with causal filtering are associated with signals satisfying
	\begin{align*}
	\textstyle \hat{C}_i={B}_i+  \sum_{j=1}^{i-1}{A}_{i,j}\hat{G}_{\pi,j} \quad \text{for }\; 2\leq i \leq k,
	\end{align*}
	where $\hat{C}_i$ and $\hat{G}_{\pi,j}$ are obtained by respectively stacking $\hat{C}_i(s,a)\in\mR^{d_i}$ and $\hat{G}_{\pi,j} (s,a)\in\mR^{d_i}$ over $\mcs\times\mca$, ${B}_i\in\mR^{d_i|\mcs||\mca|}$ is an observable signal, and the coefficient matrix ${A}_{i,j}\in \mR^{d_i|\mcs||\mca|\times d_j|\mcs||\mca|}$ captures how the $j$-th GVF $\hat{G}_{\pi,j}$ affects the $i$-th accumulation signal $\hat{C}_i$. Further, ${B}_i$ and ${A}_{i,j}$ are bounded for all $i,j=1,\cdots,k$ to ensure $\hat{G}_{\pi,i}$ to be well defined.
\end{definition}

For an ergodic MDP that starts from $-\infty$, we have $(s_{t-1},a_{t-1})\sim\mu_{\pi}(\cdot)$, $(s_t,a_t)\sim \msP_\pi(\cdot|s_{t-1},a_{t-1})$, and $(s_{t},a_{t})\sim\mu_{\pi}(\cdot)$ for all $-\infty < t < \infty$. The Bayes' theorem implies that 
\begin{flalign}
	P((s_{t-1},a_{t-1})|(s_t,a_t)) = \frac{\mu_{\pi}(s_{t-1},a_{t-1})\msP_\pi((s_t,a_t)|(s_{t-1},a_{t-1}))}{\mu_{\pi}(s_t,a_t)}.\label{eq: backward_prob}
\end{flalign}
The reverse conditional probability in \cref{eq: backward_prob} together with the definition of backward GVF in \Cref{def:backwardGVF} implies that the backward GVFs $\hat{G}_\pi = [\hat{G}^\top_{\pi,1}, \cdots, \hat{G}^\top_{\pi,k}]^\top$ with causal filtering satisfies
\begin{flalign}\label{backward_GBE}
	\hat{G}_\pi = \hat{\mct}_{G,\pi} \hat{G}_\pi = {B} + \hat{M}_\pi \hat{G}_\pi,
\end{flalign}
where $\hat{\mct}_{G,\pi}$ denotes the backward general Bellman operator, and 
\begin{flalign*}
{B} = \left[\begin{array}{c}
{B}_1\\
{B}_2\\
\vdots\\
{B}_k
\end{array}
\right],\,\,
\hat{M}_\pi = \left[\begin{array}{cccc}
\gamma_1\hat{P}_{\pi,1} & 0 & \cdots  & 0 \\
{A}_{2,1} & \gamma_2 \hat{P}_{\pi,2} & \cdots  & 0\\
\vdots & \vdots & & \vdots \\
{A}_{k,1} & {A}_{k,2} & \cdots & \gamma_k \hat{P}_{\pi,k}
\end{array}
\right],\,\text{where}\,\hat{P}_{\pi,i} = U^{-1}_{\pi,i}[\msP_\pi\otimes\msI_{d_i}]U_{\pi,i}.
\end{flalign*}


{\bf Disscusion and Applications.} GVFs with causal filtering can cover a number of important RL applications. We discuss in detail in \Cref{sc: application} to show how the variance of "reward-to-go", gradient of value function, and state/action value function fall into the framework of forward GVFs in \Cref{def:forwardGVF}, and anomaly detection and gradient of logarithmic stationary distribution fall into the framework of backward GVFs in \Cref{def:backwardGVF}, respectively.

\section{Off-Policy Evaluation of GVFs: Formulation and Algorithm}
\subsection{Problem Formulation}
In this paper, we study the GVF evaluation problem for a target policy $\pi$. We focus on the {\em behavior-agnostic off-policy} setting, in which we have access only to samples generated from an off-policy (i.e., a behavior policy) with the distribution $\mcd$, i.e., $(s_j,a_j,B_j, s^\prime_j)\sim \mcd$ $(j>0)$.
Specifically, the state-action pair $(s_j,a_j)$ is sampled from a possibly \textbf{\em unknown} distribution $D(\cdot):\mcs\times\mca\rightarrow[0,1]$, $B_j=[B_1(s_j,a_j),\cdots,B_k(s_j,a_j)]$ is an observable signal vector, and the successor state $s^\prime_i$ is sampled from $\msP(\cdot|s_i,a_i)$. 
Without loss of generality, we consider the case in which $D(s,a)>0$ for all $(s,a)\in\mcs\times\mca$.
Our goal is to design an efficient algorithm to estimate $G_\pi$ (or $\hat{G}_\pi$) given the sample set $\{(s_j,a_j,B_j,s^\prime_j)\}_{j>0}$.

\subsection{Linear Function Approximation}\label{susc: linear}
When $|\mcs|$ is large, a linear function can be used to approximate the GVF: $G_{\pi,i}(s,a)\approx G_{\pi,i}(\theta_i;s,a) = \theta^\top_i\phi_i(s,a)=[\phi_i(s,a)^\top\otimes \msI_{d_i}] \mvec(\theta^\top_i)$, where $\phi_i(s,a) \in \mR^{K_i}$ is the feature vector, and $\theta_i\in\mR^{K_i\times d_i}$ is a learnable weight matrix. In the sequel, we omit $\pi$ in $G_{\pi,i}$ and use the notation $G_i$. Without loss of generality, we assume that $\ltwo{\phi_i(s,a)}\leq 1$ for all $i=1,\cdots,k$ and $(s,a)\in\mcs\times\mca$. The linear approximation can then be written as $G_i(\theta_i)=[{\rm\Phi}_i\otimes \msI_{d_i}]\mvec(\theta_i^\top)$, where ${\rm\Phi}_i$ is the base matrix obtained by stacking $\phi_i(s,a)^\top$ over $\mcs\times\mca$. 
To ensure the uniqueness of the solution $\theta_i$, we assume that ${\rm\Phi}_i$ has linearly independent columns. The joint vector of GVFs can be denoted as $[G_1^\top(\theta_1),\cdots,G_k^\top(\theta_k)]^\top$, which is captured by the joint parameters $\theta=[\mvec(\theta_1^\top)^\top, \cdots, \mvec(\theta_k^\top)^\top]^\top\in \mR^{\sum_{i=1}^{k}K_id_i}$.
Then the function approximation of GVFs can be written more compactly as $G(\theta) = {\rm\Phi}\theta$, where ${\rm\Phi}=\diag([{\rm\Phi_1}\otimes\msI_{d_1}],\cdots, [{\rm\Phi_k}\otimes\msI_{d_k}])$. For each $(s,a)$, the linear function approximation associated with each $(s,a)$ can be written as $G(\theta;s,a)=\phi(s,a)\theta$, where $\phi(s,a)=\diag([\phi_1(s,a)^\top\otimes\msI_{d_1}],\cdots, [\phi_k(s,a)^\top\otimes\msI_{d_k}])$.
We define the linear function space spanned by the columns of the feature matrix ${\rm \Phi}$ as $\mcf_{\rm\Phi} = \{{\rm\Phi}\theta| \theta\in R_\theta \}$, in which $R_\theta$ is a convex set. Given the function class $\mcf_{\rm\Phi}$, the evaluation problem of GVFs amounts to searching for a parameter $\theta^*\in R_\theta$ such that $G(\theta^*)$ approximates $G_\pi$ (or $\hat{G}_\pi$) well. In the sequel, we use $\bar{\mct}_{G,\pi}$ to represent $\mct_{G,\pi}$ or $\hat{\mct}_{G,\pi}$, interchangeably, based on the context.

\subsection{A New Off-policy GVF Evaluation Approach and Comparison to GTD}\label{sec:gentd}

\paragraph{Drawbacks of GTD.}
%
In previous works, the gradient TD (GTD) method \cite{sutton2009fast,maei2011gradient} has been used for policy evaluation (including GVF evaluation) in the off-policy setting \cite{silver2013gradient,zhang2020provably,zhang2020learning,xu2021doubly}. GTD adopts the Mean Squared Projected Bellman Error (MSPBE) for GVF evaluation with linear function approximation, which is given by
\begin{flalign}
\hat{\theta}^*=\argmin_{\theta\in R_\theta}\text{MSPBE}(\theta)\triangleq\mE_{D}\left[\ltwo{G(\theta;s,a)-{\rm\Gamma}_{\mcf_{\rm\Phi},D}\bar{\mct}_{G,\pi}G(\theta;s,a)}^2\right].\label{eq: globalGTD}
\end{flalign}
where ${\rm\Gamma}_{\mcf_{\rm\Phi},D}$ denotes the projection operator onto the space $\mcf_{\rm\Phi}$ w.r.t. the $\lsd{\cdot}$--norm, i.e., for any vector function $f(s,a)$ of $(s,a)$, we have ${\rm\Gamma}_{\mcf_{\rm\Phi},d}f=G(\theta_f)$, in which $\theta_f = \argmin_{\theta\in R_\theta} \lsd{f - G(\theta)}$.
One drawback of GTD is that the expectation in the objective function is taken over the off-policy sampling distribution $D(\cdot)$, which does not exactly reflect the desirable evaluation under the {\bf target} policy. As the result, 
the optimal point of GTD ($\hat{\theta}^*$) can still have a large approximation error with respect to the {\bf ground truth} value of GVF, even if the approximation function class is arbitrarily expressive. More detailed discussion about GTD is provided in \Cref{sc: GTD}.

\paragraph{Generalized Temporal Difference (GenTD) Learning.}
In this work, we propose a novel unified approach to evaluate both the forward and backward GVFs in the off-policy setting, which we refer to as generalized temporal difference (GenTD) learning. Specifically, we aim to learn $\theta^*$ for GVF evaluation by minimizing the mean-squared projected general Bellman error (MSPGBE) defined as
\begin{flalign}
	\theta^*=\argmin_{\theta\in R_\theta} \text{MSPGBE}(\theta)\triangleq\mE_{\mu_\pi}\left[\ltwo{G(\theta;s,a)-{\rm\Gamma}_{\mcf_{\rm\Phi},\mu_{\pi}}\bar{\mct}_{G,\pi}G(\theta;s,a)}^2\right],\label{eq: 47}
\end{flalign}
where recall that $\bar{\mct}_{G,\pi}$ represents the GBO of either forward or backward GVFs.
In contrast to GTD, the objective function in \cref{eq: 47} takes the expectation over the stationary distribution $\mu_\pi$ of the target distribution, which precisely captures the desired goal of GVF evaluation under the target policy. On the other hand, such an objective does cause implementation challenge, because the data samples are generated by the behavior policy, so that estimators based on such data directly can incur a large bias error. To solve such an issue, we will apply the density ratio $\rho(s,a)=\mu_{\pi}(s,a)/D(s,a)$ to adjust the distribution and further adopt the GenDICE/GradientDICE method proposed in \cite{zhang2020gendice,zhang2020gradientdice} to estimate $\rho(s,a)$ during the execution of the algorithm. 

To describe our algorithm GenTD (see \Cref{algorithm_offtd}), we first note that \cref{eq: 47} implies the following optimality condition for $\theta^*$,
\begin{flalign*}
	\langle G(\theta^*;\cdot) - \bar{\mct}_{G;\pi}G(\theta^*;\cdot), f(\cdot) - G(\theta^*;) \rangle_{\mu_\pi}\geq 0,\quad \forall f\in \mcf_{\rm\Phi},
\end{flalign*}
or equivalently
\begin{flalign}
\langle g(\theta^*), \theta - \theta^* \rangle\geq 0,\quad \forall \theta\in R_\theta,\label{eq: 48}
\end{flalign}
where $g(\theta) = {\rm\Phi}^\top U_\pi (G(\theta) - \bar{\mct}_{G,\pi}G(\theta))$.
The variational inequality theory \cite[Chapter 3]{lan2020first} suggests that under an appropriately chosen stepsize $\alpha_t$, the update $\theta_{t+1} = {\rm\Gamma}_{R_\theta}(\theta_t - \alpha_t g(\theta_t))$ converges to the optimal point $\theta^*$, where ${\rm\Gamma}_{R_\theta}$ denotes the projection operator onto the set $R_\theta$ in terms of the Euclidean norm. However, since it is intractable to explicitly compute $g(\theta)$ in practice, we usually estimate $g(\theta)$ using random samples. In the off-policy setting, consider a sample $x=(s,a,s^\prime,a^\prime)$, in which $(s,a)\sim D(\cdot)$, $s^\prime\sim\msP(\cdot|s,a)$, and $a^\prime\sim \pi(\cdot|s^\prime)$, we can formulate the following update rule:
\begin{flalign}
	\theta_{t+1} = \theta_t - \alpha_t \hat{\rho}(s,a)g(x,\theta_t),\label{eq: 49}
\end{flalign}
where $\hat{\rho}(s,a)$ is an approximation of the density ratio $\rho(s,a)=\mu_{\pi}(s,a)/D(s,a)$, $g(x,\theta) = -\phi(s,a)^\top \delta(x,\theta)$ for forward GVFs and $g(x,\theta) = -\phi(s^\prime,a^\prime)^\top\delta(x,\theta)$ for backward GVFs, where $\delta(x,\theta)$ is the temporal difference error defined as $\delta(x,\theta) = B(s,a)+m(x)\phi(s^\prime,a^\prime)\theta - \phi(s,a)\theta$ for forward GVFs, and $\delta(x,\theta) = {B}(s^\prime,a^\prime)+\hat{m}(x)\phi(s,a)\theta - \phi(s^\prime,a^\prime)\theta$ for backward GVFs. Here $m$ and $\hat{m}$ are matrices that capture the correlations between difference estimations in forward and backward GVFs evaluation settings, respectively. 
Here we adopt the GenDICE/GradientDICE method that proposed in \cite{zhang2020gendice,zhang2020gradientdice} to learn $\rho(s,a)$. In previous works, GenDIC/GradientDICE has only been used for estimating the scalar value $\mE_{\mu_\pi}[r(s,a)]$ in the off-policy setting \cite{zhang2020gendice,zhang2020learning,tang2019doubly}. Our work is the first to adapt this method to solve the more challenging off-policy GVFs evaluation problem.

\paragraph{Learning Density Ratio.} GenDICE/GradientDICE estimates the density ratio $\rho(s,a)$ via solving the following min-max problem \cite{zhang2020gendice,zhang2020gradientdice}:
{\begin{flalign}\label{eq: loss_gendice2}
\min_{\rho}\max_{f,\eta} L(\hat{\rho},f,\eta) \coloneqq \mE_\mcd[\hat{\rho}(f^\prime-f)] - \frac{1}{2}\mE_\mcd[f^2] + \mE_\mcd[\eta\hat{\rho}-\eta] - \frac{1}{2}\eta^2.
\end{flalign}}
We parameterize both $\rho$ and $f$ by linear function approximation with linearly independent features $\psi\in\mR^{d_\rho}$, i.e., $\hat{\rho}(s,a;w_\rho)=\psi(s,a)^\top w_\rho$ and $\hat{f}(s,a;w_f)=\psi(s,a)^\top w_f$ for all $(s,a)\in \mcs\times\mca$. To guarantees the stability of the density ratio learning, we assume that the matrix $A = \mE_{\mcd\cdot \pi}[\psi(\psi-\psi^\prime)^\top]$ is non-singular. Note that this assumption can be removed by adding an $l_2$--regularizer in \cref{eq: loss_gendice2}.

In GenTD (see \Cref{algorithm_offtd}), we estimate the density ratio via updating the parameter $w_{\rho,t}$ iteratively. The density estimator $\hat{\rho}(s_t,a_t;w_{\rho_t})=\psi(s_t,a_t)^\top w_{\rho_t}$ is then used to reweight the update $g(x_t,\theta_t)$.  As we will show in the next section, even though the estimation of $\rho(s,a)$ is not always accurate during the training, \Cref{algorithm_offtd} can still converge to $\theta^*$. 

\textbf{Comparison between GenTD and GTD.} Compared with GTD, our GenTD has the following two advantages. First, since GTD does not adjust the distribution mismatch of sampling, the optimal point of GTD can suffer from large approximation error with respect to the ground truth GVFs even with highly expressive function classes. In contrast, the optimum of GenTD is guaranteed to converge to the ground truth GVFs with sufficiently expressive function classes. Second, GTD needs to update a high-dimensional auxiliary parameter $w$ simultaneously with $\theta$ to stabilize the convergence, where $w\in \mR^{\sum_{i=1}^k K_id_i}$ has the same dimension as $\theta$ (note that $d_i$ can be large in the high dimensional regime). Such an update of $w$ can be very costly. In contrast, GenTD introduces only low-dimensional auxiliary parameters $[w_\rho,w_f, \eta]\in\mR^{2d_\rho+1}$ for density ratio estimation, which is more efficient. 

\begin{algorithm}[tb]
	\caption{Generalized TD Learning (GenTD)}
	\label{algorithm_offtd}
	\begin{algorithmic}
		\STATE {\bfseries Initialize:} Approximator parameters $w_{f,0}$, $w_{\rho,0}$ and $\theta_0$
		\FOR{$t=0,\cdots,T-1$}
		\STATE Obtain sample $(s_t,a_t,C_t,s^\prime_t)\sim \mcd_d$ and $a^\prime_t\sim \pi(\cdot|s^\prime_t)$
		\STATE $\bar{\delta}_{t} = \psi_t^\top\theta_{\rho,t}(\psi^\prime_t - \psi_t)$
		\STATE $\eta_{t+1} = w_{\rho,t} + \beta_t (\psi_t^\top w_{\rho,t} - 1 - \eta_t)$
		\STATE $w_{f,t+1} = w_{f,t} + \beta_t(\bar{\delta}_{t} - \psi^\top_tw_{f,t}\psi_t)$
		\STATE $w_{\rho,t+1} ={\rm\Gamma}_{R_\rho}\left(w_{\rho,t} - \beta_t( \psi^{\prime\top}_tw_{f,t}\psi_t - \psi_t^\top w_{f,t}\psi_t + \eta_t\psi_t )\right)$
		\STATE $\theta_{t+1} = {\rm\Gamma}_{R_\theta}\left(\theta_t - \alpha_t [w_{\rho,t}^\top \psi(s_t,a_t)]g(x_t,\theta_t)\right)$ 
		\STATE $\quad${\bf Forward GVF:} $g(x,\theta)=-\phi(s,a)^\top(B(s,a)+m(x)\phi(s^\prime,a^\prime)\theta - \phi(s,a)\theta)$
		\STATE $\quad${\bf Backward GVF:} $g(x,\theta)=-\phi(s^\prime,a^\prime)^\top({B}(s^\prime,a^\prime)+\hat{m}(x)\phi(s,a)\theta - \phi(s^\prime,a^\prime)\theta)$
		\ENDFOR
	\end{algorithmic}
\end{algorithm}

\section{Main Theorems}\label{sec:mainresults}

In this section, we characterize the convergence rate and optimality guarantee for GenTD. To this end, we first establish a contraction property for the general Bellman operator (GBO) of interest here. Although the contraction property has been proven in the canonical value function settings \cite{tsitsiklis1997analysis,zhang2020learning}, it is unclear whether such a property still holds for {\bf multiple interrelated and high-dimensional} GVFs. We will next show that only under a properly chosen norm, such a property holds for both forward and backward GVFs with causal filtering. This is the first result of such a type.

Consider the GVFs vector $G_\pi = [G^\top_{\pi,1}, \cdots, G^\top_{\pi,k}]^\top$. We define a norm $\lalpha{\cdot}$ associated with a weighting vector $\alpha=[\alpha_1,\cdots,\alpha_k]\in\Delta_k$, where $\Delta_k$ denotes the simplex in $k$-dimensional space, as 
\begin{flalign}
{\textstyle \lalpha{G_\pi} = \sum_{i=1}^{k}\alpha_i \lsmu{G_{\pi,i}}\,\text{where}\,\,\,  0<\alpha_i\leq 1 \quad\text{for all}\,\,\, i \,\,\, \text{and} \,\,\, \sum_{i=1}^{k}\alpha_i=1}.\label{eq: 54}
\end{flalign}
We also define $\gamma_{\max} := \max_{i=1\cdots,k}\gamma_i$, which is strictly less than $1$.
\begin{proposition}[Contraction of Forward/Backward GBO]\label{cond1}
	For any ${G}_{\pi}, {G}^\prime_{\pi} \in\mR^{|\mcs||\mca|\sum_{i=1}^{k}K_id_i}$, there exists a weighting vector $\alpha$ such that
	\begin{flalign}
		\lalpha{\bar{\mct}_{G,\pi}{G}_{\pi} - \bar{\mct}_{G,\pi}{G}^\prime_{\pi}}\leq \gamma_G\lalpha{{G}_{\pi} - {G}^\prime_{\pi}},\label{eq: 52}
	\end{flalign}
	where $\gamma_G=(1+\gamma_{\max})/2$ and $\bar{\mct}_{G,\pi}$ can be either $\mct_{G,\pi}$ (forward GBO, \cref{forward_GBE}) or $\hat{\mct}_{G,\pi}$ (backward GBO, \cref{backward_GBE}).
\end{proposition}
Despite the correlations between GVFs, \Cref{cond1} shows that the contraction property is still preserved under a properly chosen norm for ${\mct}_{G,\pi}$ and $\hat{\mct}_{G,\pi}$ in forward and backward GVF settings, respectively.  The norm can vary for different GVFs.
\Cref{cond1} also implies that both forward and backward GVFs ($G_\pi$ and $\hat{G}_\pi$) can be identified as unique fixed point of their corresponding GBOs. 

Based on \Cref{cond1}, we next establish the monotonicity property for our GenTD algorithm, if it takes the population update $g(\theta)= {\rm\Phi}^\top U_\pi (G(\theta) - \bar{\mct}_{G,\pi}G(\theta))$.
\begin{proposition}[Monotonicity]\label{cond2}
Consider the globally optimal point $\theta^*$ defined in \cref{eq: 48}. There exists a constant $\lambda_G$ such that for all $\theta\in R_\theta$, we have
	\begin{flalign}
	\langle  g(\theta^*)-g(\theta), \theta^*-\theta \rangle\geq \lambda_G\lF{\theta-\theta^*}^2,\label{eq: 55}
	\end{flalign}
	where $\lambda_G := (1-\gamma_{\max})\min_{1\leq i\leq k}\zeta_i$ and $\zeta_i := \lambda_{\min}({\rm\Phi}^\top_iU_\pi{\rm\Phi}_i)$.
\end{proposition}
\Cref{cond2} implies the contraction property of $g(\theta)$. It guarantees that $\theta$ moves towards a globally optimal point $\theta^*$ if it is updated along the direction $-g(\theta)$. \Cref{cond2} generalizes the monotonicity property to a much broader class of interrelated and multi-dimensional GVF evaluation, which is far more beyond TD learning for the value function evaluation studied in \cite{tsitsiklis1997analysis,zhang2020learning}.
The following theorem characterizes the convergence rate of GenTD.
\begin{theorem}\label{thm1}
	Consider the GenTD update in \Cref{algorithm_offtd}. Let the stepsize $\alpha_t={\rm\Theta}(t^{-1})$ and $\beta_t = {\rm\Theta}(t^{-1})$. We have
	\begin{flalign}
	\mE[\lF{\theta_{T}-\theta^*}^2]\leq \mathcal{O}\left(\frac{\lF{\theta_0-\theta^*}^2}{T^2}\right) + \mathcal{O}\left(\frac{1}{\lambda^3_GT}\right) + \mathcal{O}\left(\frac{\varepsilon_\rho}{\lambda^2_G}\right),\label{eq: 51}
	\end{flalign}
	where $\varepsilon_\rho = \sqrt{\mE_{\mcd\cdot \pi}[\hat{\rho}(s,a;w^*_\rho)  - \rho(s,a)]^2 }$ is the approximation error introduced by the density ratio learning, with $w^*_\rho$ being the global optimal point of $L(\hat{\rho},f,\eta)$ defined in \cref{eq: loss_gendice2}.
\end{theorem}
\Cref{thm1} shows that GenTD converges to the globally optimal point $\theta^*$ at a rate $\mathcal{O}(1/T)$. The convergence speed of $\theta$ also depends on the conditional number $\lambda_G$, where the converge becomes faster as $\lambda_G$ increases. Specifically, the R.H.S. of \cref{eq: 51} consists of three terms. The first term corresponds to the initialization error, which delays as fast as  $\mathcal{O}(1/T^2)$. The second term corresponds to the variance error, which dominates the convergence rate of GenTD to be $\mathcal{O}(1/T)$. The last term corresponds to a non-vanishing optimality gap, which is introduced by the function approximation error in the density ratio estimation, and decreases as the expressive power of the approximation function class $\{\hat{\rho}(w_\rho): w_\rho \in R_\rho\}$ increases. The convergence analysis of GenTD is more challenging than that of TD learning \cite{bhandari2018finite,dalal2018finite,srikant2019finite} and GTD \cite{xu2019two,kaledin2020finite}, as we need to handle an additional approximation error introduced by the dynamically changing density ratio estimator $\hat{\rho}(w_{\rho_t})$.



\Cref{thm1} establishes the convergence of GenTD to the globally optimal point $\theta^*$ of the objective function in \cref{eq: 47}, which provides the value estimation $G(\theta^*)$ for the GVFs. We are then interested in characterizing how close such an estimation is to the ground truth GVF $G_\pi$, which is our ultimate goal of evaluation. We characterize this in the following theorem.
\begin{theorem}[Convergence of GenTD to Ground Truth]\label{thm3}
	Consider $\theta^*$ defined in \cref{eq: 47}. Suppose the same conditions in \Cref{cond1} and \Cref{cond2} hold. We have
	\begin{flalign}
		\lalpha{ G(\theta^*) - G_\pi } \leq \frac{1}{1-\gamma_G}\lalpha{{\rm\Gamma}_{\mcf_{\rm\Phi},\mu_{\pi}} G_\pi - G_\pi}.\label{eq: 60}
	\end{flalign}
\end{theorem}
\Cref{thm3} indicates that the distance between the optimal estimation $G(\theta^*)$ and the true GVF $G_\pi$ is upper bounded by the approximation error of the function class $\mcf_{\rm\Phi}$ for the ground truth GVF $G_\pi$ (note that ${\rm\Gamma}_{\mcf_{\rm\Phi},\mu_{\pi}} G_\pi$ denotes the projection of $G_\pi$ to the function approximation class $\mcf_{\rm\Phi}$). Hence, \Cref{thm3} guarantees that $G(\theta^*)$ can be as close as possible to the true GVF $G_\pi$, as long as the function class $\mcf_{\rm\Phi}$ is sufficiently expressive. In particular, if $\mcf_{\rm\Phi}$ is complete, i.e., there exists $G_\theta\in\mcf_{\rm\Phi}$ such that $G_\theta=G_\pi$, then GenTD is guaranteed to converge exactly to the ground truth $G_\pi$. 

{\bf Comparison between GenTD and GTD.} If $\mcf_{\rm\Phi}$ is complete, GTD performs similarly to GenTD and is guaranteed to converge to the ground truth $G_\pi$ (see \Cref{sc: quality_gtd} for the proof). The major difference between GenTD and GTD occurs when $\mcf_{\rm\Phi}$ is not complete. In such a case, our GenTD still maintains the desirable performance as guaranteed by \Cref{thm3}, but the optimal point of GTD (i.e., $\hat{\theta}^*$ in \cref{eq: globalGTD}) does not have guaranteed convergence to the ground truth. As shown in \cite{kolter2011fixed,hallak2017consistent,munos2003error}, even in the value function evaluation setting (a special case of forward GVF evaluation) the approximation error $||G(\hat{\theta}^*)-G_\pi||_D$ of GTD can be arbitrarily poor even if $\mcf_{\rm\Phi}$ can represent the true value function arbitrarily well (but not exactly). Such a disadvantage of GTD is mainly due to the distribution mismatch in its objective function as we discuss in \Cref{sec:gentd}.

In the backward GVFs evaluation setting, GTD can perform even worse. As we show in the following example, GTD may fail to learn the ground truth $G_\pi$ even if the function class $\mcf_{\rm\Phi}$ is complete. Note that for such a case, GenTD converges to the ground truth as guaranteed by \Cref{thm3}.

\begin{example}[GTD Fails for Complete $\mcf_{\rm\Phi}$]\label{ex1}
	Consider a three-state Markov chain, with transition kernel $\msP = [[0.1, 0.9, 0], [0.1, 0, 0.9], [0, 0.1, 0.9]]^\top$, discount factor $\gamma=0.99$, and the reward function $R = [1, 0, 1]^\top$. The back value function in this MDP is given by $\bar{V}= [8.1555, 9.0389, 9.0184]^\top$. Suppose GTD is applied to solving the evaluation problem with the parameter space $R_\theta = \mR$. Then, there exists an off-policy distribution $D$ such that using the perfect bases ${\rm\Phi} = [8.1555, 9.0389, 9.0184]^\top$, the optimal point $\bar{\theta}^*$ learned by GTD still has non-zero approximation error, i.e., $||{\rm\Phi}\bar{\theta}^*-\bar{V}||_D\geq 3$.
\end{example}

\vspace{-0.2cm}
\section{Experiments}\label{sc: exp}
We conduct empirical experiments to answer the following two questions: (a) can GenTD evaluate both the forward and backward GVFs efficiently? (2) how does GenTD compare with GTD in terms of the convergence speed and the quality of the estimation results? 
\begin{figure}[h]
	\centering 
	\includegraphics[width=1.9in]{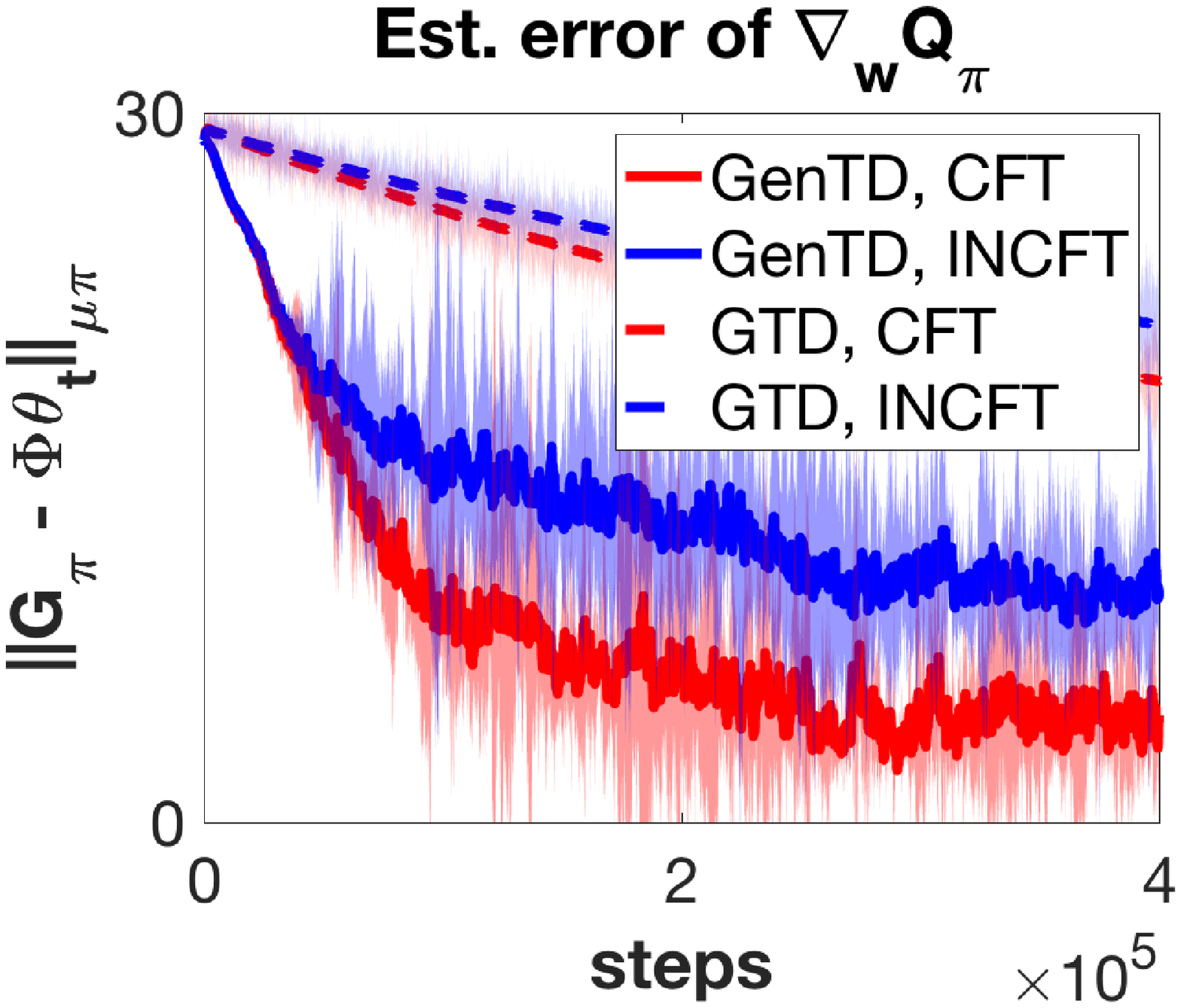}\includegraphics[width=1.9in]{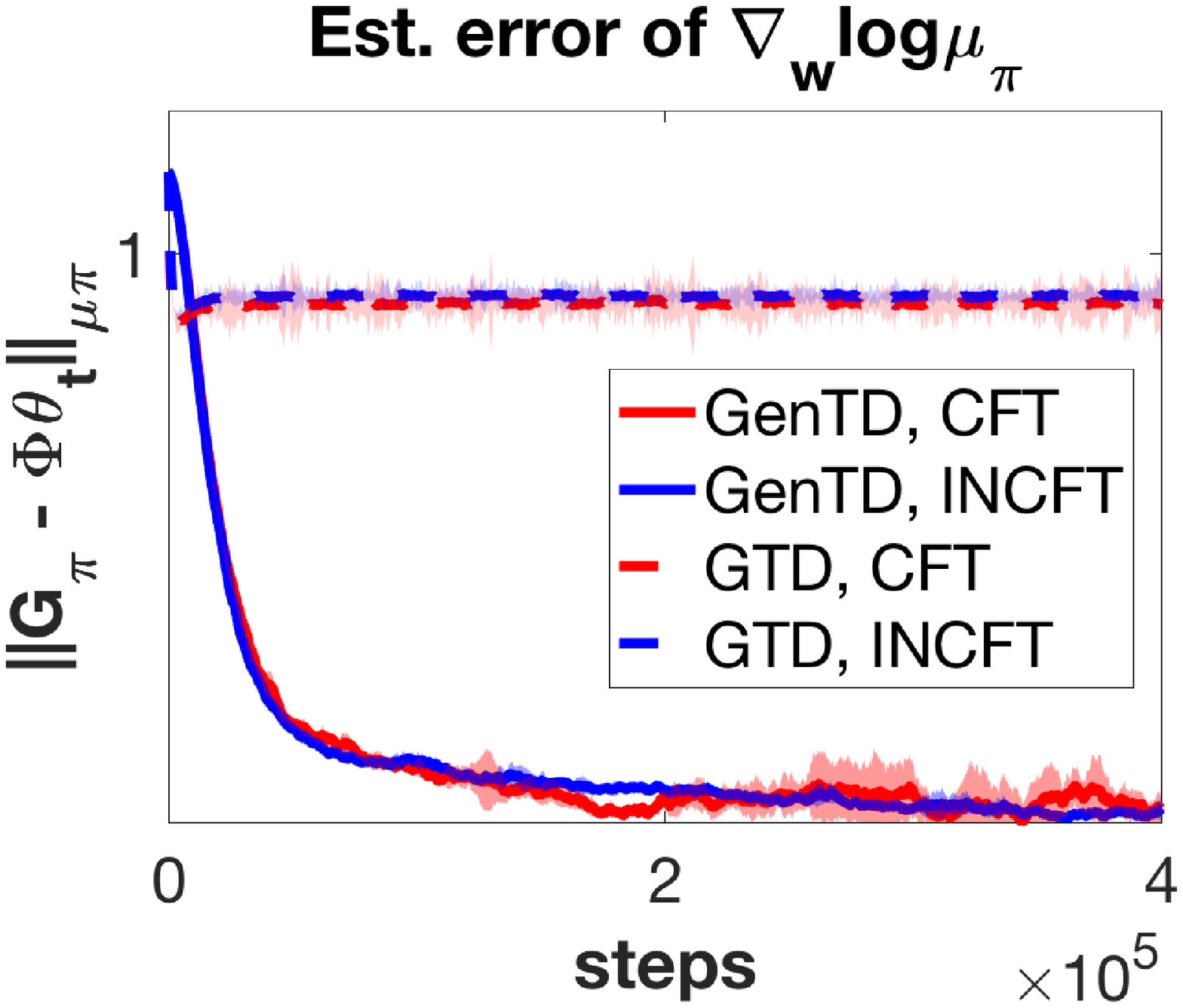}
	\caption{\small\em Comparison between GenTD and GTD for the tasks of evaluating $\nabla_w Q_\pi$ and $\nabla_w \log\mu_\pi$. }
	\label{fig: 1}
	\vspace{-0.2cm}
\end{figure}
In our experiments, we consider a variant of Baird's counterexample \cite{baird1995residual,sutton2018reinforcement} with 7 states and 2 actions (see \Cref{fig: baird} in \Cref{sc: app_exp}). We study the problem of evaluating two high-dimensional GVFs, the gradient of Q-function: $\nabla_w Q_\pi \in\mR^{14}$ ({\rm forward GVF}), and the gradient of logarithmic stationary distribution: $\nabla_w\log(\mu_\pi)\in\mR^{14}$ (backward GVF), associated with a soft-max policy parameterized by $w\in\mR^{14}$. We consider two types of feature matrices ${\rm\Phi}$ for estimating the GVFs: complete feature (CFT) and incomplete feature (INCFT), where CFT has large enough expressive power so that the ground true GVF can be fully expressed by the function class $\mcf_{\rm\Phi}$, whereas INCFT does not have enough expressive power and cannot capture the ground true GVF exactly. The discount factor $\gamma$ is set to be $0.99$ in all tasks, and all curves in the plots are averaged over 20 independent runs. The detailed experimental setting is provided in \Cref{sc: app_exp}.

The learning curves for GenTD and GTD are provided in \Cref{fig: 1}. We evaluate their performances based on the estimation error with respect to the ground truth GVF: $\lsmu{{\rm\Phi}\theta_t-{G}_\pi}$. Note that both $\nabla_w Q_\pi$ and $\nabla_w \log\mu_\pi$ can be exactly computed in this tabular setting, so that the estimator error of the ground truth can be computed. For the task of $\nabla_w Q$ evaluation, GenTD converges considerably faster and much closer to the ground truth (i.e., smaller estimation error) than GTD.
which can be attributed to the larger conditional number $\lambda_G$ of GenTD. For the task of $\nabla_w \log\mu_\pi$ evaluation, GenTD moves fast towards the ground truth GVF, whereas GTD, although still converges,
stays far away from the ground truth GVF even with CFT, which matches with our \Cref{ex1}. As we discuss in \Cref{sec:mainresults}, this is because
GTD in the backward GVF evaluation setting has distribution mismatch in its objective function, which can significantly shift the optimal point from the ground truth GVF.

\vspace{-0.2cm}
\section{Conclusion}
We studied the off-policy evaluation problem of both forward and backward GVFs. We focused on the class of GVFs with casual filtering, which covers a wide range of multiple interrelated and possibly high-dimensional GVFs. We first showed that GVFs in such a class is the fixed point of a general Bellman operator. Based on such a property, we proposed a new off-policy algorithm called GenTD. GenTD evaluates GVFs efficiently by jointly updating the GVF approximation parameter and a density ratio estimator, which adjusts the mismatch of the behavior policy and assists the convergence to the ground truth GVFs. We show that GenTD provably converges to the globally optimal point, and such an optimal point is guaranteed to converge to the ground truth GVFs as long as the function expressive power is sufficiently large. For future work, it is interesting to study nonlinear function approximation for GVFs evaluation.

\section{Acknowledgement}
The work of T. Xu and Y. Liang was supported in part by the U.S. National Science Foundation under the grants CCF-1801855, CCF-1761506 and CCF-1900145.

\bibliographystyle{abbrv}
\bibliography{ref}

\begin{thebibliography}{10}

\bibitem{baird1995residual}
L.~Baird.
\newblock Residual algorithms: reinforcement learning with function
  approximation.
\newblock In {\em Machine Learning Proceedings}, pages 30--37. 1995.

\bibitem{bhandari2018finite}
J.~Bhandari, D.~Russo, and R.~Singal.
\newblock A finite time analysis of temporal difference learning with linear
  function approximation.
\newblock In {\em Proc. Conference on Learning Theory (COLT)}, pages
  1691--1692, 2018.

\bibitem{cai2019neural}
Q.~Cai, Z.~Yang, J.~D. Lee, and Z.~Wang.
\newblock Neural temporal-difference and q-learning provably converge to global
  optima.
\newblock In {\em Proc. Neural Information Processing Systems (NeurIPS)}, 2019.

\bibitem{chandak2021universal}
Y.~Chandak, S.~Niekum, B.~C. da~Silva, E.~Learned-Miller, E.~Brunskill, and
  P.~S. Thomas.
\newblock Universal off-policy evaluation.
\newblock {\em arXiv preprint arXiv:2104.12820}, 2021.

\bibitem{comanici2018knowledge}
G.~Comanici, D.~Precup, A.~Barreto, D.~K. Toyama, E.~Ayg{\"u}n, P.~Hamel,
  S.~Vezhnevets, S.~Hou, and S.~Mourad.
\newblock Knowledge representation for reinforcement learning using general
  value functions.
\newblock {\em OpenReview}, 2018.

\bibitem{dalal2020tale}
G.~Dalal, B.~Szorenyi, and G.~Thoppe.
\newblock A tale of two-timescale reinforcement learning with the tightest
  finite-time bound.
\newblock In {\em Proc. AAAI Conference on Artificial Intelligence (AAAI)},
  volume~34, pages 3701--3708.

\bibitem{dalal2018finite}
G.~Dalal, B.~Sz{\"o}r{\'e}nyi, G.~Thoppe, and S.~Mannor.
\newblock Finite sample analyses for {TD} (0) with function approximation.
\newblock In {\em Proc. AAAI Conference on Artificial Intelligence (AAAI)},
  2018.

\bibitem{d2020learn}
P.~D'Oro and W.~Ja{\'s}kowski.
\newblock How to learn a useful critic? model-based action-gradient-estimator
  policy optimization.
\newblock {\em arXiv preprint arXiv:2004.14309}, 2020.

\bibitem{downey2017predictive}
C.~Downey, A.~Hefny, B.~Li, B.~Boots, and G.~Gordon.
\newblock Predictive state recurrent neural networks.
\newblock In {\em Proc. Neural Information Processing Systems (NeurIPS)}, pages
  6055--6066, 2017.

\bibitem{hallak2017consistent}
A.~Hallak and S.~Mannor.
\newblock Consistent on-line off-policy evaluation.
\newblock In {\em Proc. International Conference on Machine Learning (ICML)},
  pages 1372--1383, 2017.

\bibitem{hazan2019provably}
E.~Hazan, S.~Kakade, K.~Singh, and A.~Van~Soest.
\newblock Provably efficient maximum entropy exploration.
\newblock In {\em International Conference on Machine Learning}, pages
  2681--2691, 2019.

\bibitem{heess2015learning}
N.~Heess, G.~Wayne, D.~Silver, T.~Lillicrap, Y.~Tassa, and T.~Erez.
\newblock Learning continuous control policies by stochastic value gradients.
\newblock {\em arXiv preprint arXiv:1510.09142}, 2015.

\bibitem{horn2012matrix}
R.~A. Horn and C.~R. Johnson.
\newblock {\em Matrix analysis}.
\newblock Cambridge University Press, 2012.

\bibitem{huang2020importance}
J.~Huang and N.~Jiang.
\newblock From importance sampling to doubly robust policy gradient.
\newblock In {\em Proc. International Conference on Machine Learning (ICML)},
  pages 4434--4443, 2020.

\bibitem{jain2021variance}
A.~Jain, G.~Patil, A.~Jain, K.~Khetarpal, and D.~Precup.
\newblock {V}ariance penalized on-policy and off-policy actor-critic.
\newblock {\em arXiv preprint arXiv:2102.01985}, 2021.

\bibitem{jiang2016doubly}
N.~Jiang and L.~Li.
\newblock Doubly robust off-policy value evaluation for reinforcement learning.
\newblock In {\em Proc. International Conference on Machine Learning (ICML)},
  pages 652--661, 2016.

\bibitem{kaledin2020finite}
M.~Kaledin, E.~Moulines, A.~Naumov, V.~Tadic, and H.-T. Wai.
\newblock Finite time analysis of linear two-timescale stochastic approximation
  with markovian noise.
\newblock In {\em Proc. Conference on Learning Theory (COLT)}, pages
  2144--2203, 2020.

\bibitem{kallus2020statistically}
N.~Kallus and M.~Uehara.
\newblock Statistically efficient off-policy policy gradients.
\newblock {\em arXiv preprint arXiv:2002.04014}, 2020.

\bibitem{kolter2011fixed}
J.~Kolter.
\newblock The fixed points of off-policy {TD}.
\newblock In {\em Proc. Advances in Neural Information Processing Systems
  (NIPS)}, volume~24, pages 2169--2177, 2011.

\bibitem{lan2020first}
G.~Lan.
\newblock {\em First-order and Stochastic Optimization Methods for Machine
  Learning}.
\newblock Springer, 2020.

\bibitem{liu2015finite}
B.~Liu, J.~Liu, M.~Ghavamzadeh, S.~Mahadevan, and M.~Petrik.
\newblock Finite-sample analysis of proximal gradient {TD} algorithms.
\newblock In {\em UAI}, pages 504--513, 2015.

\bibitem{maei2011gradient}
H.~R. Maei.
\newblock {\em Gradient temporal-difference learning algorithms}.
\newblock PhD thesis, University of Alberta, 2011.

\bibitem{mahmood2013representation}
A.~R. Mahmood and R.~S. Sutton.
\newblock Representation search through generate and test.
\newblock In {\em Proc. AAAI Workshop: Learning Rich Representations from
  Low-Level Sensors}, 2013.

\bibitem{mannor2013algorithmic}
S.~Mannor and J.~N. Tsitsiklis.
\newblock Algorithmic aspects of mean--variance optimization in markov decision
  processes.
\newblock {\em European Journal of Operational Research}, 231(3):645--653,
  2013.

\bibitem{morimura2010derivatives}
T.~Morimura, E.~Uchibe, J.~Yoshimoto, J.~Peters, and K.~Doya.
\newblock Derivatives of logarithmic stationary distributions for policy
  gradient reinforcement learning.
\newblock {\em Neural computation}, 22(2):342--376, 2010.

\bibitem{munos2003error}
R.~Munos.
\newblock Error bounds for approximate policy iteration.
\newblock In {\em Proc. International Conference on Machine Learning (ICML)},
  volume~3, pages 560--567, 2003.

\bibitem{ruszczynski2010risk}
A.~Ruszczy{\'n}ski.
\newblock Risk-averse dynamic programming for markov decision processes.
\newblock {\em Mathematical programming}, 125(2):235--261, 2010.

\bibitem{sato2001td}
M.~Sato, H.~Kimura, and S.~Kobayashi.
\newblock {TD} algorithm for the variance of return and mean-variance
  reinforcement learning.
\newblock {\em Transactions of the Japanese Society for Artificial
  Intelligence}, 16(3):353--362, 2001.

\bibitem{schaul2013better}
T.~Schaul and M.~Ring.
\newblock Better generalization with forecasts.
\newblock In {\em Proc. International Joint Conference on Artificial
  Intelligence (IJCAI)}, pages 1656--1662, 2013.

\bibitem{sharpe1966mutual}
W.~F. Sharpe.
\newblock Mutual fund performance.
\newblock {\em The Journal of business}, 39(1):119--138, 1966.

\bibitem{sherstan2020gradient}
C.~Sherstan.
\newblock {\em Representation and general value functions}.
\newblock PhD thesis, University of Alberta, 2020.

\bibitem{shortreed2011informing}
S.~M. Shortreed, E.~Laber, D.~J. Lizotte, T.~S. Stroup, J.~Pineau, and S.~A.
  Murphy.
\newblock Informing sequential clinical decision-making through reinforcement
  learning: an empirical study.
\newblock {\em Machine Learning}, 84(1-2):109--136, 2011.

\bibitem{silver2013gradient}
D.~Silver.
\newblock Gradient temporal difference networks.
\newblock In {\em European Workshop on Reinforcement Learning}, pages 117--130,
  2013.

\bibitem{silver2014deterministic}
D.~Silver, G.~Lever, N.~Heess, T.~Degris, D.~Wierstra, and M.~Riedmiller.
\newblock Deterministic policy gradient algorithms.
\newblock In {\em Proc. International Conference on Machine Learning (ICML)},
  pages 387--395, 2014.

\bibitem{sobel1982variance}
M.~J. Sobel.
\newblock The variance of discounted {Markov} decision processes.
\newblock {\em Journal of Applied Probability}, pages 794--802, 1982.

\bibitem{srikant2019finite}
R.~Srikant and L.~Ying.
\newblock Finite-time error bounds for linear stochastic approximation andtd
  learning.
\newblock In {\em Proc. Conference on Learning Theory (COLT)}, pages
  2803--2830, 2019.

\bibitem{sun2016learning}
W.~Sun, A.~Venkatraman, B.~Boots, and J.~A. Bagnell.
\newblock Learning to filter with predictive state inference machines.
\newblock In {\em Proc. International Conference on Machine Learning (ICML)},
  pages 1197--1205, 2016.

\bibitem{sutton1988learning}
R.~S. Sutton.
\newblock Learning to predict by the methods of temporal differences.
\newblock {\em Machine {Learning}}, 3(1):9--44, 1988.

\bibitem{sutton2009grand}
R.~S. Sutton.
\newblock The grand challenge of predictive empirical abstract knowledge.
\newblock In {\em Proc. IJCAI Workshop on Grand Challenges for Reasoning from
  Experiences}, 2009.

\bibitem{sutton2018reinforcement}
R.~S. Sutton and A.~G. Barto.
\newblock {\em Reinforcement {Learning}: An introduction}.
\newblock MIT press, 2018.

\bibitem{sutton2009fast}
R.~S. Sutton, H.~R. Maei, D.~Precup, S.~Bhatnagar, D.~Silver,
  C.~Szepesv{\'a}ri, and E.~Wiewiora.
\newblock Fast gradient-descent methods for temporal-difference learning with
  linear function approximation.
\newblock In {\em Proc. International Conference on Machine Learning (ICML)},
  pages 993--1000, 2009.

\bibitem{sutton2016emphatic}
R.~S. Sutton, A.~R. Mahmood, and M.~White.
\newblock An emphatic approach to the problem of off-policy temporal-difference
  learning.
\newblock {\em The Journal of Machine Learning Research}, 17(1):2603--2631,
  2016.

\bibitem{sutton2000policy}
R.~S. Sutton, D.~A. McAllester, S.~P. Singh, and Y.~Mansour.
\newblock Policy gradient methods for reinforcement learning with function
  approximation.
\newblock In {\em Proc. Advances in Neural Information Processing Systems
  (NeurIPS)}, pages 1057--1063, 2000.

\bibitem{sutton2011horde}
R.~S. Sutton, J.~Modayil, M.~Delp, T.~Degris, P.~M. Pilarski, A.~White, and
  D.~Precup.
\newblock {H}orde:{ A} scalable real-time architecture for learning knowledge
  from unsupervised sensorimotor interaction.
\newblock In {\em Proc. International Conference on Autonomous Agents and
  Multiagent Systems}, pages 761--768, 2011.

\bibitem{sutton2004temporal}
R.~S. Sutton and B.~Tanner.
\newblock Temporal-difference networks.
\newblock volume~17, pages 1377--1384, 2004.

\bibitem{tamar2012policy}
A.~Tamar, D.~Di~Castro, and S.~Mannor.
\newblock Policy gradients with variance related risk criteria.
\newblock In {\em Proc. International Conference on Machine Learning (ICML)},
  pages 1651--1658, 2012.

\bibitem{tamar2016learning}
A.~Tamar, D.~Di~Castro, and S.~Mannor.
\newblock Learning the variance of the reward-to-go.
\newblock {\em The Journal of Machine Learning Research}, 17(1):361--396, 2016.

\bibitem{tamar2013variance}
A.~Tamar and S.~Mannor.
\newblock Variance adjusted actor critic algorithms.
\newblock {\em arXiv preprint arXiv:1310.3697}, 2013.

\bibitem{tang2019doubly}
Z.~Tang, Y.~Feng, L.~Li, D.~Zhou, and Q.~Liu.
\newblock Doubly robust bias reduction in infinite horizon off-policy
  estimation.
\newblock {\em arXiv preprint arXiv:1910.07186}, 2019.

\bibitem{tsitsiklis1997analysis}
J.~N. Tsitsiklis and B.~Van~Roy.
\newblock Analysis of temporal-diffference learning with function
  approximation.
\newblock In {\em Proc. Advances in Neural Information Processing Systems
  (NeurIPS)}, pages 1075--1081, 1997.

\bibitem{tsitsiklis1997average}
J.~N. Tsitsiklis and B.~Van~Roy.
\newblock Average cost temporal-difference learning.
\newblock In {\em Proc. IEEE Conference on Decision and Control}, volume~1,
  pages 498--502, 1997.

\bibitem{white2015developing}
A.~White et~al.
\newblock Developing a predictive approach to knowledge.
\newblock 2015.

\bibitem{xu2021sample}
T.~Xu and Y.~Liang.
\newblock Sample complexity bounds for two timescale value-based reinforcement
  learning algorithms.
\newblock In {\em Proc. International Conference on Artificial Intelligence and
  Statistics (AISTATS)}, pages 811--819, 2021.

\bibitem{xu2020improving}
T.~Xu, Z.~Wang, and Y.~Liang.
\newblock Improving sample complexity bounds for (natural) actor-critic
  algorithms.
\newblock {\em In Proc. Advances in Neural Information Processing Systems
  (NeurIPS)}, 33, 2020.

\bibitem{xu2020non}
T.~Xu, Z.~Wang, and Y.~Liang.
\newblock Non-asymptotic convergence analysis of two time-scale (natural)
  actor-critic algorithms.
\newblock {\em arXiv preprint arXiv:2005.03557}, 2020.

\bibitem{xu2021doubly}
T.~Xu, Z.~Yang, Z.~Wang, and Y.~Liang.
\newblock Doubly robust off-policy actor-critic: convergence and optimality,
  2021.

\bibitem{xu2019two}
T.~Xu, S.~Zou, and Y.~Liang.
\newblock Two time-scale off-policy {TD} learning: Non-asymptotic analysis over
  markovian samples.
\newblock In {\em Proc. Advances in Neural Information Processing Systems
  (NeurIPS)}, pages 10633--10643, 2019.

\bibitem{yao2013reinforcement}
H.~Yao and D.~Schuurmans.
\newblock Reinforcement ranking.
\newblock {\em arXiv preprint arXiv:1303.5988}, 2013.

\bibitem{zhang2020variational}
J.~Zhang, A.~Koppel, A.~S. Bedi, C.~Szepesvari, and M.~Wang.
\newblock Variational policy gradient method for reinforcement learning with
  general utilities.
\newblock {\em arXiv preprint arXiv:2007.02151}, 2020.

\bibitem{zhang2020gendice}
R.~Zhang, B.~Dai, L.~Li, and D.~Schuurmans.
\newblock Gendice: generalized offline estimation of stationary values.
\newblock {\em arXiv preprint arXiv:2002.09072}, 2020.

\bibitem{zhang2019generalized}
S.~Zhang, W.~Boehmer, and S.~Whiteson.
\newblock Generalized off-policy actor-critic.
\newblock In {\em Proc. Advances in Neural Information Processing Systems
  (NeurIPS)}, pages 2001--2011, 2019.

\bibitem{zhang2020gradientdice}
S.~Zhang, B.~Liu, and S.~Whiteson.
\newblock {G}radient{DICE}: rethinking generalized offline estimation of
  stationary values.
\newblock In {\em Proc. International Conference on Machine Learning (ICML)},
  pages 11194--11203, 2020.

\bibitem{zhang2019provably}
S.~Zhang, B.~Liu, H.~Yao, and S.~Whiteson.
\newblock Provably convergent off-policy actor-critic with function
  approximation.
\newblock In {\em Proc. International Conference on Machine Learning (ICML)},
  2020.

\bibitem{zhang2020provably}
S.~Zhang, B.~Liu, H.~Yao, and S.~Whiteson.
\newblock Provably convergent two-timescale off-policy actor-critic with
  function approximation.
\newblock In {\em Proc. International Conference on Machine Learning (ICML)},
  pages 11204--11213, 2020.

\bibitem{zhang2020learning}
S.~Zhang, V.~Veeriah, and S.~Whiteson.
\newblock Learning retrospective knowledge with reverse reinforcement learning.
\newblock {\em Proc. Advances in Neural Information Processing Systems
  (NeurIPS)}, 33, 2020.

\end{thebibliography}


\newpage

\appendix
\noindent {\Large \textbf{Supplementary Materials}}
\section{Specification of Experiments}\label{sc: app_exp}
The Baird's counterexample \cite{baird1995residual,sutton2018reinforcement} is shown in \Cref{fig: baird}. There are two actions represented by solid line and dash line, respectively. The the {\em dash} action leads to states 1-6 with equal probability and a reward $+1$, and {\em solid} action always leads to state 7 and a reward $0$. The behavior distribution over the state-action space $(s,a)$ is given as
\begin{flalign*}
D(\cdot)=
\begin{cases}
   D(s_1,a_1) = 0.2,\quad &D(s_1,a_2) = 0.1,\\
   D(s_2,a_1) = 0.2,\quad &D(s_2,a_2) = 0.1,\\
   D(s_3,a_1) = 0.04,\quad &D(s_3,a_2) = 0.04,\\
   D(s_4,a_1) = 0.04,\quad &D(s_4,a_2) = 0.04,\\
   D(s_5,a_1) = 0.04,\quad &D(s_5,a_2) = 0.04,\\
   D(s_6,a_1) = 0.04,\quad &D(s_6,a_2) = 0.04,\\
   D(s_7,a_1) = 0.04,\quad &D(s_7,a_2) = 0.04,
\end{cases}
\end{flalign*}
\begin{figure}[ht]
	\vskip 0.2in
	\begin{center}
		\centerline{\includegraphics[width=60mm]{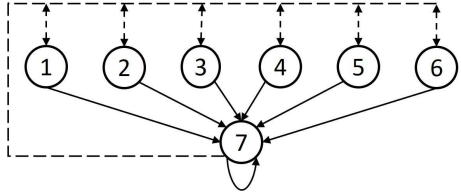}}
		\caption{A variant of Baird's counterexample.}
		\label{fig: baird}
	\end{center}
	\vskip -0.2in
\end{figure}
where $s_i$ denotes state "$i$" ($i=1,\cdots,7$), $a_1$ denotes the {\em dash} action and $a_2$ denotes the {\em solid} action. We consider the soft-max policy given as
\begin{flalign*}
   \pi_w(s_i,a_j) = \frac{\exp(w_{2(i-1)+j})}{\exp(w_{2(i-1) + 1}) + \exp(w_{2(i-1)+2})},\quad i=\{1,\cdots,7\},\,\, j=\{1,2\},
\end{flalign*}
where $w\in\mR^{14}$ is the parameter of the policy given as
\begin{flalign*}
   w^\top = [0.0, 1.8, 0.0, 1.8, 0.0, 1.8, 0.0, 1.8, 0.0, 1.8, 0.0, 1.8, 0.0, 1.8].
\end{flalign*}
The complete feature (CFT), incomplete feature (INCFT) and the learning rate for each task are given as follows:
\begin{itemize}
    \item {\bf Evaluation of $\nabla_w Q_\pi$ (forward GVF).} In this task, we need to evaluate both $Q_\pi$ and $\nabla_w Q_\pi$ (see \Cref{sc: app_forward} for detailed discussion about correlation between $Q_\pi$ and $\nabla_w Q_\pi$). Let the complete feature matrix ${\rm\Phi}$ be the identity matrix, i.e., ${\rm\Phi}=I\in\mR^{14\times 14}$. We let CFT for each $(s,a)$ be one of the rows of ${\rm\Phi}$. We further remove one column of ${\rm\Phi}$ to obtain the incomplete feature matrix ${\rm\Phi}^\prime\in\mR^{14\times 13}$. We let INCFT for each $(s,a)$ be one of the rows of ${\rm\Phi}^\prime$. For GenTD, the learning rate for updating $w_\rho$ and $\theta$ are $0.01$ and $0.005$, respectively. We use the same CFT and INCFT for the density ratio estimation as those for the $\nabla_w Q_\pi$ estimation. For GTD, the learning rate for both the main parameter $\theta$ and the auxiliary parameter $w$ are $0.005$.
    \item {\bf Evaluation of $\nabla_w \log\mu_\pi$ (backward GVF).} \Cref{sc: app_backward} has detailed discussion about such a GVF. Note that this task corresponds to the setting where $\gamma=1$. As discussed in \Cref{sc: extension}, the ground true GVF in this setting is in the space perpendicular to the vector $e=[1,\cdots,1]^\top\in\mR^{14}$. Here we use singular value decomposition (SVD) to obtain the complete feature matrix ${\rm\Phi}\in\mR^{14\times 13}$ such that ${\rm\Phi}\theta\neq Ce$ for all $\theta\neq 0$ and $C\neq 0$. We let CFT for each $(s,a)$ be one of the rows of ${\rm\Phi}$. We further remove one column of ${\rm\Phi}$ to obtain the incomplete feature matrix ${\rm\Phi}^\prime\in\mR^{14\times 12}$. We let INCFT for each $(s,a)$ be one of the rows of ${\rm\Phi}^\prime$. For GenTD, the learning rate for updating $w_\rho$ and $\theta$ are $0.05$ and $0.005$, respectively. We use the same CFT and INCFT for the density ratio estimation as those used in the above $\nabla_w Q_\pi$ estimation task. For GTD, the learning rate for both the main parameter $\theta$ and the auxiliary parameter $w$ are $0.005$.
\end{itemize}

\section{Examples of Forward and Backward GVFs}\label{sc: application}
In this section, we present a number of example forward and backward GVF in RL applications.
\subsection{Examples of Forward GVFs}\label{sc: app_forward}
The forward GVF in \Cref{def:forwardGVF} arises naturally in the following RL applications. 

{\bf Case I: Variance of Reward-To-Go. } In risk-sensitive domains such as finance, process control and clinical decision making \cite{sharpe1966mutual,mannor2013algorithmic,tamar2013variance,tamar2012policy,ruszczynski2010risk,sato2001td,sobel1982variance,jain2021variance}, in addition to the mean $J_\pi$ of the "reward-to-go", we are also interested in the variance of $J_\pi$ \cite{sharpe1966mutual,shortreed2011informing}, which is given by $\text{Var}[J_\pi|(s_0,a_0)=(s,a)]= H_\pi(s,a) - Q^2_\pi(s,a)$, where $H_\pi$ is the second moment of $J_\pi$, i.e., $H_\pi(s,a)=\mE[J_\pi^2|(s_0,a_0)=(s,a),\pi]$. \cite{tamar2016learning} shows that $H_\pi$ satisfies
\begin{flalign}
H_\pi = R^2 + 2\gamma M_R \msP_\pi Q_\pi + \gamma^2 \msP_\pi H_\pi,\label{eq: 32}
\end{flalign}
where $R$ is defined in \Cref{sc: background}, $M_R = \diag(R)\in\mR^{|\mcs||\mca|\times |\mcs||\mca|}$. \Cref{eq: 32} implies that $H_\pi$ is the mean of the accumulation of signal $C(s,a)=r(s,a)^2 + 2\gamma \mE[Q_\pi(s^\prime,a^\prime)|s,a]$ with discounted factor $\gamma^2$. Since $C(s,a)$ is a function of the reward $r(s,a)$ and value function $Q_\pi(s,a)$, we consider the joint vector of $Q_\pi$ and $H_\pi$ as $G_\pi$, i.e., $G_\pi = [Q^\top_\pi, H^\top_\pi]^\top$. We have that $G_\pi$ satisfies the general Bellman equation in \cref{forward_GBE} with $B$ and $M_\pi$ specified as
\begin{flalign}\label{eq: variance}
B = \left[\begin{array}{c}
R\\
R^2
\end{array}
\right],\quad
M_\pi = \left[\begin{array}{cc}
\gamma\msP_\pi & 0 \\
2\gamma M_R\msP_\pi & \gamma^2 \msP_\pi
\end{array}
\right].
\end{flalign}
We consider the setting in which reward is bounded, i.e., $r(s,a)\leq C_R$ for all $(s,a)\in\mcs\times\mca$. 

{\bf Case II: Gradient of Q-function.} Suppose that the policy is parametrized by a smooth function $\pi_w$, in which $w\in\mR^{d_w}$ is the parameter. Then the gradient $\nabla_w Q_{\pi}(s,a)$ of the Q-function w.r.t $w$ plays an important role in several RL applications such as variance reduced policy gradient \cite{huang2020importance} and on- and off-policy policy optimization \cite{xu2021doubly,silver2014deterministic,kallus2020statistically,comanici2018knowledge}. Specifically, \cite{xu2021doubly,kallus2020statistically,comanici2018knowledge} show that $\nabla_wQ_\pi$ satisfies:
\begin{flalign}
\nabla_w Q_{\pi} = \gamma [\msP_\pi\otimes \msI_{d_w}] [\nabla_w{\rm\Pi}_{\pi_w}\cdot Q_{\pi}] + \gamma [\msP_\pi \otimes \msI_{d_w}]\nabla_wQ_{\pi},\label{eq: 33}
\end{flalign}
where $\nabla_w{\rm\Pi}_{\pi_w}\in\mR^{d_w|\mcs||\mca|}$ is obtained by stacking $\nabla_w\log(\pi(s,a))$ over $\mcs\times\mca$, i.e., $[\nabla_w{\rm\Pi}_{\pi_w}](s,a)=\nabla_w\log(\pi_w(s,a))$, and $[\nabla_w{\rm\Pi}_{\pi_w}\cdot Q_{\pi}]\in\mR^{d_w|\mcs||\mca|}$ is  element-wise product between $\nabla_w{\rm\Pi}_{\pi_w}$ and $Q_\pi$, i.e., $[\nabla_w{\rm\Pi}_{\pi_w}\cdot Q_{\pi}](s,a)=\nabla_w\log(\pi_w(s,a))Q_{\pi}(s,a)$. \Cref{eq: 33} implies that $\nabla_wQ_\pi$ is the mean of the accumulation of signal $C(s,a)=\gamma \mE[Q_\pi(s^\prime,a^\prime)\nabla_w\log(\pi_w(s^\prime,a^\prime))|s,a]$ with the discounted factor $\gamma$. Let $G_\pi = [Q^\top_\pi, \nabla_wQ^\top]^\top$. We have that $G_\pi$ satisfies the general Bellman equation in \cref{forward_GBE} with $B$ and $M_\pi$ specified as
\begin{flalign}\label{eq: dq}
B = \left[\begin{array}{c}
R\\
0
\end{array}
\right],\quad
M_\pi = \left[\begin{array}{cc}
\gamma\msP_\pi & 0 \\
\gamma[\msP_\pi\otimes \msI_{d_w}] \diag(\nabla_w{\rpii}_{\pi_w}) & \gamma\msP_\pi\otimes\msI_{d_w}
\end{array}
\right],
\end{flalign}
where $\diag(\nabla_w{\rpii}_{\pi_w}) \in\mR^{d_w|\mcs||\mca|\times |\mcs||\mca|}$ is obtained by arranging $\nabla_w\log(\pi(s,a))\in\mR^{d_w}$ diagonally. Without loss of generality, we assume that the score function is bounded \cite{sutton2018reinforcement,sutton2000policy,xu2020improving}, i.e., $\ltwo{\nabla_w\log(\pi_w(s,a))}\leq C_{\rm\Pi}$ for all $(s,a)\in\mcs\times\mca$. 

{\bf Case III: Stochastic Value Gradient.} The stochastic value gradient (SVG) method combines advantages of model-based and model-free methods, in which both the estimated model and value function are updated to evaluate the policy gradient \cite{heess2015learning}. In the framework of SVG, the reward $r(s,a)$ is differentiable with respect to both $s\in\mR^{d_s}$ an $a\in\mR^{d_a}$, the stochastic policy takes the form $a=\pi(s,w)+\eta$, and the transition probability is modelled as $s^\prime = f(s,a)+\xi$, where $\pi: \mcs\rightarrow \mca$ and $f: \mcs\times\mca\rightarrow \mcs$ are deterministic mappings, $w\in\mR^{d_w}$ is the policy parameter, and $\eta\sim P(\eta)$ and $\xi\sim P(\xi)$ are noise variables. We abbreviate the partial differentiation using subscripts as $g_x\triangleq \partial g/\partial x$. The gradient of the $Q$-function w.r.t the policy parameter $w$ is given by \cite{heess2015learning}
\begin{flalign}
Q_s &= (R_s + {\rm\Pi}_s R_a) + \gamma ({\rm\Pi}_s F_a + F_s)([\msP_\pi\otimes \msI_{d_s}] Q_s),\nonumber\\
Q_a &= R_a + \gamma F_a [\msP_\pi\otimes \msI_{d_s}] Q_s + \gamma [\msP_\pi\otimes \msI_{d_a}]Q_a,\nonumber\\
\nabla_wQ_\pi &= {\rm \Pi}_w Q_a + \gamma [\msP_\pi\otimes \msI_{d_w}]\nabla_wQ_\pi,\label{eq: 41}
\end{flalign}
where $R_s\in\mR^{d_s|\mcs||\mca|}$ and $R_a \in\mR^{d_a|\mcs||\mca|}$ are vectors obtained via stacking partial derivatives $r_s(s,a) \in\mR^{d_s}$ and $r_a(s,a)\in\mR^{d_a}$ over $(s,a)\in\mcs\times\mca$, and ${\rm\Pi}_s=\diag(\pi_s)\in \mR^{d_s|\mcs||\mca|\times d_a|\mcs||\mca|}$, ${\rm\Pi}_w=\diag(\pi_w)\in \mR^{d_w|\mcs||\mca|\times d_a|\mcs||\mca|}$, $F_s=\diag(f_s)\in \mR^{d_s|\mcs||\mca|\times d_s|\mcs||\mca|}$, and $F_a=\diag(f_a)\in \mR^{d_a|\mcs||\mca|\times d_s|\mcs||\mca|}$ are Jacobian matrices. 
Consider GVF defined as $G_\pi=[Q^\top_s, Q^\top_a, \nabla_wQ^\top_\pi]^\top$. Consider the normalized setting in which ${\rm\Pi}_s F_a + F_s=\mathbf{I}$. Then $G_\pi$ satisfies the general Bellman equation in \cref{forward_GBE} with $B$ and $M_\pi$ specified by
\begin{flalign}
B &= \left[\begin{array}{c}
R_s + {\rm\Pi}_s R_a\\
R_a\\
{\rm\Pi}_w Q_a
\end{array}
\right],\quad
M_\pi = \left[\begin{array}{ccc}
\gamma[\msP_\pi\otimes \msI_s] & 0 & 0  \\
\gamma F_a [\msP_\pi\otimes \msI_s] & \gamma[\msP_\pi\otimes \msI_a] &  0 \\
0 & {\rm\Pi}_w & \gamma [\msP_\pi\otimes \msI_{d_w}]
\end{array}
\right].\label{eq: svg}
\end{flalign}
We consider the setting in which $\lF{f_a(s,a)}\leq C_a$ and $\lF{\pi_w(s,a)}\leq C_w$ for all $(s,a)\in\mcs\times\mca$.

\subsection{Examples of Backward GVFs}\label{sc: app_backward}
The backward GVF in \Cref{def:backwardGVF} also arises in the following important RL applications.

{\bf Case IV: Anomaly Detection.} \cite{zhang2020learning} has systemically discussed the application of retrospective knowledge in anomaly detection. Let $i(s,a)$ be the cost that an agent consumes when taking action $a$ at state $s$, and $e_\pi(s,a)$ be the cost that an agent is expected to consume given the current status when following a predefined policy $\pi$. If the actual cost of the agent deviates too much from $e_\pi$, the agent may likely encounter anomalous events. For simplicity, we consider the setting when $\gamma(s,a)=\gamma$. It can be shown that $e_\pi$ satisfies the following equation
\begin{flalign}
e_\pi =   i +  \gamma U^{-1}_\pi \msP^\top_\pi U_\pi e_\pi.\label{eq: 34}
\end{flalign}
Clearly, \cref{eq: 34} satisfies the general backward Bellman equation in \cref{backward_GBE} by letting $B=i$, $M_\pi = \gamma\msP_\pi$, and $\hat{G}_\pi = e_\pi$. 

{\bf Case V: Gradient of Logarithmic Stationary Distributions.} In the policy parameterization setting, the gradient of logarithmic stationary distribution $\nabla_w \log\mu_{\pi}(s,a)$ has been used in policy gradient estimation \cite{kallus2020statistically,xu2021doubly,morimura2010derivatives} and maximum entropy exploration \cite{hazan2019provably}. It has been shown in \cite{morimura2010derivatives,xu2021doubly} that $\nabla_w \log\mu_{\pi}(s,a)$ satisfies the following equation
\begin{flalign}
{\rm\Psi}_\pi = \nabla_w{\rm\Pi}_{\pi_w} + U^{-1}_\pi[\msP_\pi^\top \otimes \msI_{d_w}] U_\pi{\rm\Psi}_\pi , \label{eq: 35}
\end{flalign}
where ${\rm\Psi}_\pi $ is obtained via stacking $\nabla_w\log(\mu_{\pi}(s,a))$ over $\mcs\times\mca$, i.e., $[{\rm\Psi}_\pi](s,a) = \nabla_w\log(\mu_{\pi_w}(s,a))$. Here, $\nabla_w\log\mu_\pi$ can be viewed as a backward accumulation of the signal $C(s,a)=\nabla_w \log(\pi(s,a))$ with the discounted factor $\gamma=1$. Define the backward GVF as $\hat{G}_\pi={\rm\Psi}_\pi$. It is clear that $\hat{G}_\pi$ satisfies the general backward Bellman equation in \cref{backward_GBE} with $B$ and $M_\pi$ specified by
\begin{flalign}\label{eq: 46}
B = \nabla_w{\rm\Pi}_{\pi_w},\qquad M_\pi = \msP_\pi^\top \otimes \msI_{d_w}.
\end{flalign}
Note that since $\gamma_{\max}=1$ in the general Bellman equation in \cref{eq: 35}, the result in \Cref{cond1} may not hold in such a setting, i.e., GBO may not be a contraction here. However, as we will show in \cref{sc: extension}, when the base matrix ${\rm\Phi}$ satisfies the "non-constant parameterization" assumption, we can establish results similar to \Cref{cond2} and \Cref{thm2} for the evaluation of $\nabla_w\log\mu_\pi$.

\section{Gradient Temporal Difference Learning (GTD)}\label{sc: GTD}
The GTD algorithm has been used for GVF evaluation in \cite{sutton2011horde,silver2013gradient}. So far, only the asymptomatic convergence (not the convergence rate) has been studied in \cite{silver2013gradient}.
In this section, we present the GTD algorithm for GVF evaluation and characterize the finite-time convergence rate for GTD. We define the dimension of parameter $\theta$ as $d_g = \sum_{i=1}^k K_id_i$. In the sequel, we denote the MSPBE with parameter $\theta$ as $J(\theta)$.
Note that the MSPBE in \cref{eq: globalGTD} can be rewritten as
\begin{flalign}
J(\theta) & = \frac{1}{2}\lcin{\mE_D[g(x,\theta)]}^2,\,\, C\triangleq \begin{cases} \mE_D[(\phi(s,a)^\top\phi(s,a))]  \quad &\text{(forward GVF)}\\ \textstyle{\mE_D[(\phi(s^\prime,a^\prime)^\top\phi(s^\prime,a^\prime))]} & \text{(backward GVF)} \end{cases}\label{eq: 21},
\end{flalign}
where
\begin{flalign*}
    g(x,\theta)\triangleq \begin{cases}
    -\phi(s,a)^\top(B(s,a)+m(x)\phi(s^\prime,a^\prime)\theta - \phi(s,a)\theta) &\text{(forward GVF)}\nonumber\\
    -\phi(s^\prime,a^\prime)^\top({B}(s^\prime,a^\prime)+\hat{m}(x)\phi(s,a)\theta -\phi(s^\prime,a^\prime)\theta) &\text{(backward GVF)}\end{cases},
\end{flalign*}
where the matrices $m(\cdot)$ and $\hat{m}(\cdot)$ are defined in \cref{eq: 49}.
The gradient of $J(\theta)$ is given as
\begin{flalign*}
-\nabla J(\theta) = \mE_D[g(x,\theta) - h(x,\theta)],\,\,h(x,\theta)\triangleq \begin{cases} \phi(s^\prime,a^\prime)^\top m(x)\phi(s,a)w(\theta) &\text{(forward GVF)}\\ \textstyle{\phi(s,a)^\top \hat{m}(x) \phi(s^\prime,a^\prime)w(\theta)} & \text{(backward GVF)} \end{cases}
\end{flalign*}
in which $w(\theta) = C^{-1} \mE_D[g(x,\theta)]\in\mR^{d_g}$. Note that we can not estimate $\nabla J(\theta)$ directly due to the "double sampling" issue, i.e., $w(\theta)$ cannot be estimated via sampling. In GTD, an auxiliary parameter $w_t$ is introduced, which is updated simultaneously with $\theta_t$ to approximate $w(\theta_t)$ \cite{sutton2009fast,maei2011gradient}. We present the update of GTD in \Cref{algorithm_gtd}.

\begin{algorithm}[tb]
	\caption{GTD}
	\label{algorithm_gtd}
	\begin{algorithmic}
		\STATE {\bfseries Initialize:} Approximator parameters $w_0$, and $\theta_0$
		\FOR{$t=0,\cdots,T-1$}
		\STATE Obtain sample $(s_t,a_t,B_t,s^\prime_t)\sim \mcd$ and $a^\prime_t\sim \pi(\cdot|s^\prime_t)$
		\STATE $w_{t+1} = w_t - \beta_t( g(x_t,\theta_t) -  l(x_t,w_t)  )$
		\STATE $\qquad${\bf forward GVF:} $l(x_t,w_t) = \phi(s_t,a_t)^\top \phi(s_t,a_t)w_t$
		\STATE $\qquad${\bf backward GVF:} $l(x_t,w_t) = \phi(s^\prime_t,a^\prime_t)^\top \phi(s^\prime_t,a^\prime_t)w_t$
		\STATE $\theta_{t+1} = {\rm\Gamma}_{R_\theta}\left(\theta_t + \alpha_t (g(x_t,\theta_t) - h(x_t,\theta_t)) \right)$
		\STATE $\qquad${\bf forward GVF:} $ h(x_t,w_t) = \phi(s^\prime_t,a^\prime_t)^\top m(x_t) \phi(s_t,a_t)w_t $
		\STATE $\qquad${\bf backward GVF:} $ h(x_t,w_t) = \phi(s_t,a_t)^\top \hat{m}(x_t) \phi(s^\prime_t,a^\prime_t)w_t $
		\ENDFOR
	\end{algorithmic}
\end{algorithm}

\subsection{Convergence Rate of GTD}\label{sc: gtd}
We make the following assumptions, which have also been adopted in the convergence analysis of GTD in the canonical value function evaluation setting \cite{xu2019two,xu2020non,sutton2009fast,maei2011gradient}.
\begin{assumption}\label{ass4}
	In both forward and backward GVF evaluation settings, the matrix $C$ in \cref{eq: 21} is non-singular.
\end{assumption}
\begin{assumption}\label{ass5}
	We define the matrix $\hat{A}$ in the following way: (1) in the forward GVF evaluation setting: $\hat{A}=\mE_{\mcd\cdot\pi}[\phi(s,a)^\top(m(x)\phi(s^\prime,a^\prime) - \phi(s,a))]$; (2) in the backward GVF evaluation setting, $\hat{A}=\mE_{\mcd\cdot\pi}[\phi(s^\prime,a^\prime)(\hat{m}(x)\phi(s,a) - \phi(s^\prime,a^\prime))]$. We require $\hat{A}$ to be non-singular in both the forward and backward GVF settings.
\end{assumption}

We define the optimal point $\bar{\theta}^*$ for GTD as 
\begin{flalign*}
\langle \nabla J(\bar{\theta}^*), \theta - \bar{\theta}^* \rangle \geq 0,\quad \forall\theta\in R_\theta,
\end{flalign*}
which is the optimality condition for minimizing $J(\theta)$.
The following theorem characterizes the convergence rate of GTD to $\bar{\theta}^*$.
\begin{theorem}\label{thm2}
	Consider the GTD update in \Cref{algorithm_gtd}. In both the forward and backward GVF evaluation settings, suppose \Cref{ass4}-\ref{ass5} hold. Let the stepsize $\alpha_t = {\rm \Theta}(t^{-1})$ and $\beta_t = {\rm \Theta}(t^{-1})$. We have
	\begin{flalign*}
	\mE\left[\ltwo{{\theta}_{T}-\bar{\theta}^*}^2\right]\leq \mathcal{O}\left(\frac{\ltwo{\theta_0-\bar{\theta}^*}^2}{T^2}\right) + \mathcal{O}\left(\frac{1}{\lambda_G^{\prime2} T}\right),
	\end{flalign*}
	where $\lambda^\prime_G>0$ is the conditional number of GTD defined in \cref{eq: cond_gtd} of \Cref{sc: GTD_proof}.
\end{theorem}
\Cref{thm1} shows that GTD converges to the globally optimal point $\bar{\theta}^*$ at a rate of $\mathcal{O}(1/T)$. The convergence speed of $\theta_t$ depends on the conditional number $\lambda^\prime_G$, which decreases as $\lambda^\prime_G$ decreases. Differently from the conditional number $\lambda_G$ of GenTD, which has a guaranteed lower bound from zero as given in \Cref{cond2}, there exists no guaranteed lower bound for $\lambda^\prime_G$ even in the canonical value function evaluation setting. Thus, the converge speed of GTD could be very slow as $\lambda^\prime_G$ could be arbitrarily small.

\subsection{Global Optimum of GTD and Proof of \Cref{ex1}}\label{sc: quality_gtd}

For simplicity, we consider scenarios when the function approximation class $\mcf_{\rm\Phi}$ is complete. We show that the global optimum of GTD exhibits very different properties in the forward and backward GVF evaluation settings.

We first show that in the forward GVF evaluation setting, the global optimum ${\rm\Phi}\bar{\theta}^*$ of GTD equals the ground truth GVF. Since the function space $\mcf_{\rm\Phi}$ is complete, there exists a parameter $\theta_{\text{true}}\in R_\theta$ such that ${\rm\Phi}\theta_{\text{true}} = G_\pi$, which implies $J(\theta_{\text{true}}) = 0$. Since $J(\theta)\geq 0$ for all $\theta\in R_\theta$ and $J(\theta)$ is strongly-convex, $J(\theta)=0$ if and only if $\theta=\bar{\theta}^*$, which implies $\theta_{\text{true}}=\bar{\theta}^*$.


In the backward GVFs evaluation setting, we provide an example (see \Cref{ex1} in \Cref{sec:mainresults}) to show that GTD can fail to learn the ground truth $G_\pi$ even if the function class $\mcf_{\rm\Phi}$ is complete. We next present the proof for such an example.
\begin{proof}[{\bf Proof of \Cref{ex1}}]
    The backward value function can be obtained as follows
\begin{flalign*}
	\bar{V} = U^{-1}(I - \gamma \msP^\top)^{-1} \msP^\top U R = [8.1555, 9.0389, 9.0184]^\top.
\end{flalign*}
The fixed point of GTD is given by
\begin{flalign*}
	\bar{\theta}^* = \bar{A}^{-1}\bar{b},
\end{flalign*}
where
\begin{flalign}
	\bar{A} = \gamma {\rm\Phi}^\top \msP^\top \bar{D}^\prime {\rm\Phi} - {\rm\Phi}^\top \bar{D}^\prime {\rm\Phi},\qquad \bar{b} = {\rm\Phi}^\top\msP^\top \bar{D} R,\label{eq: 30}
\end{flalign}
where $\bar{D}=\text{diag}(D)$, $\bar{D}^\prime = \text{diag}(D^\prime)$ and $D^{\prime\top} = D^\top P$. Also note that the base matrix ${\rm\Phi} = [8.1555, 9.0389, 9.0184]^\top$ and the off-policy sampling distribution $D = [1/3,1/3, 1/3]^\top$. We can obtain
\begin{flalign*}
	\bar{A} = -9.9422,\qquad \bar{b} = 5.9904,
\end{flalign*}
which implies
\begin{flalign}
	\bar{\theta}^* = \frac{-b}{A} = 0.6025. \label{eq: 29}
\end{flalign}
Note that the perfect base matrix $[8.1555, 9.0389, 9.0184]^\top$ can fully represent $\bar{V}_\pi$, with parameter $\theta_{\text{true}} = 1$. However, \cref{eq: 29} shows that the global optimum of GTD $\bar{\theta}^*\neq \theta_{\text{true}}$, which introduces a non-zero approximation error:
\begin{flalign*}
	\lsd{{\rm\Phi}\bar{\theta}^*-\bar{V}} = 3.7848.
\end{flalign*}
\end{proof}

The above example demonstrates a drawback of GTD, which can fail to learn the ground truth $G_\pi$ even if the function class $\mcf_{\rm\Phi}$ is complete. Such an issue does not occur for the GenTD algorithm proposed in this paper, which converges to the ground truth as guaranteed by \Cref{thm3} in \Cref{sec:mainresults}.

\section{Proofs of Propositions \ref{cond1} and \ref{cond2}}

\subsection{Supporting Lemmas}

We provide the following lemmas, which are useful for the proofs of Propositions \ref{cond1} and \ref{cond2}.
\begin{lemma}\label{lemma13}
	For any $v\in\mR^{d|\mcs||\mca|}$,we have $\lmu{[\msP_\pi\otimes\msI_d]v}\leq \lmu{v}$.
\end{lemma}
\begin{proof}
	Consider the square of $\lmu{[\msP_\pi\otimes\msI_d]v}$. We have
	\begin{flalign*}
	\lmu{[\msP_\pi\otimes\msI_d]v}^2 &= v^\top [\msP^\top_\pi\otimes\msI_d]U_\pi [\msP_\pi\otimes\msI_d]v\nonumber\\
	&= \sum_{i=1}^{|\mcs||\mca|}\mu_\pi(i)  \ltwo{\sum_{j=1}^{|\mcs||\mca|} \msP_\pi(j|i) v_j }^2\nonumber\\
	&\leq \sum_{i=1}^{|\mcs||\mca|}\mu_\pi(i)  \sum_{j=1}^{|\mcs||\mca|} \msP_\pi(j|i) \ltwo{v_j}^2\nonumber\\
	&= \sum_{i=1}^{|\mcs||\mca|} \sum_{j=1}^{|\mcs||\mca|} \mu_\pi(i) \msP_\pi(j|i) \ltwo{v_j}^2\nonumber\\
	&= \sum_{j=1}^{|\mcs||\mca|} \mu_\pi(i) \ltwo{v_j}^2\nonumber\\
	& = \lmu{v}^2,
	\end{flalign*}
	where the first inequality follows from Jensen's inequality and the fourth equality follows from the property of the stationary distribution $\mu^\top_\pi\msP_\pi = \mu^\top_\pi$.
\end{proof}

We provide the follow lemma to characterize a similar property in backward GVF evaluation setting.
\begin{lemma}\label{lemma14}
	For any $v\in\mR^{d|\mcs||\mca|}$,we have $\lmu{U^{-1}_\pi [\msP^\top_\pi\otimes\msI_d] U_\pi v}\leq \lmu{v}$.
\end{lemma}
\begin{proof}
	Consider the square of $\lmu{U^{-1}_\pi [\msP^\top_\pi\otimes\msI_d] U_\pi v}$. We have
	\begin{flalign*}
	\lmu{U^{-1}_\pi [\msP_\pi\otimes\msI_d] U_\pi v}^2 &= v^\top [U^{-1}_\pi [\msP^\top_\pi\otimes\msI_d] U_\pi]^\top U_\pi [U^{-1}_\pi [\msP^\top_\pi\otimes\msI_d] U_\pi]v\nonumber\\
	&=\sum_{j=1}^{|\mcs||\mca|}\mu_{\pi}(j)\ltwo{\sum_{i=1}^{|\mcs||\mca|}\frac{\mu_{\pi}(i)\msP_\pi(j|i)}{\mu_{\pi}(j)}v(i)}^2\nonumber\\
	&\leq \sum_{j=1}^{|\mcs||\mca|}\mu_{\pi}(j)\sum_{i=1}^{|\mcs||\mca|}\frac{\mu_{\pi}(i)\msP_\pi(j|i)}{\mu_{\pi}(j)}\ltwo{v(i)}^2\nonumber\\
	&= \sum_{j=1}^{|\mcs||\mca|}\mu_{\pi}(i)\ltwo{v(i)}^2\sum_{i=1}^{|\mcs||\mca|}\msP_\pi(j|i)\nonumber\\
	&= \sum_{j=1}^{|\mcs||\mca|}\mu_{\pi}(i)\ltwo{v(i)}^2\nonumber\\
	&= \lmu{v}^2,
	\end{flalign*}
	where the first inequality follows from the Jensen's inequality.
\end{proof}

\subsection{Proof of \Cref{cond1}}
We first consider the {\bf forward GBO setting}. Recall the following definition of GBO $\mct_{G,\pi}$ in \cref{forward_GBE}
\begin{flalign*}
G_\pi = \mct_{G,\pi} G_\pi = B + M_\pi G_\pi,
\end{flalign*}
where
\begin{flalign*}
B = \left[\begin{array}{c}
B_1\\
B_2\\
\vdots\\
B_k
\end{array}
\right],\quad
M_\pi = \left[\begin{array}{cccc}
\gamma_1[\msP_\pi\otimes\msI_{d_1}] & 0 & \cdots  & 0 \\
A_{2,1} & \gamma_2 [\msP_\pi\otimes\msI_{d_2}] & \cdots  & 0\\
\vdots & \vdots & & \vdots \\
A_{k,1} & A_{k,2} & \cdots & \gamma_k [\msP_\pi\otimes\msI_{d_k}] 
\end{array}
\right].
\end{flalign*}
Let ${G}^\prime_{\pi}, {G}^{\prime\prime}_{\pi} \in\mR^{|\mcs||\mca|\sum_{i=1}^{k}K_id_i }$ be two vectors, and let $\Delta_G = [\Delta_1, \cdots, \Delta_k]$, where $\Delta_i=G^\prime_{\pi,i} - G^{\prime\prime}_{\pi,i}$. We have
\begin{flalign}
\mct_\pi G_\pi^\prime - \mct_\pi G^{\prime\prime}_\pi =  M_\pi \Delta_G = \left[\begin{array}{c}
\gamma_1[\msP_\pi\otimes\msI_{d_1}] \Delta_1 \\
A_{2,1}\Delta_1 + \gamma_2 [\msP_\pi\otimes\msI_{d_2}] \Delta_2\\
\vdots\\
\sum_{j=1}^{k-1} A_{k,j}\Delta_j +  \gamma_k [\msP_\pi\otimes\msI_{d_k}] \Delta_k
\end{array}
\right].\label{eq: 58}
\end{flalign}

Recall that $A_{i,j}$ is bounded for all $i,j$. Thus, there exists a constant $0<C_A<\infty$ such that $\lmu{A_{i,j}}\leq C_A$ for all $i,j$. Without loss of generality, we assume $C_A>1$. Let $\alpha$ be the solution of the following matrix function
\begin{flalign}
	Fx = f,\label{eq: 56}
\end{flalign}
where $F\in\mR^{k\times k}$ and $f\in\mR^k$ are specified as
\begin{flalign*}
	F = \left[\begin{array}{cccccc}
	-\frac{1-\gamma}{2} & C_A & C_A & \cdots  & C_A & C_A \\
	0 & -\frac{1-\gamma}{2} & C_A  & \cdots & C_A & C_A \\
	\vdots & \vdots & & \vdots \\
	0 & 0 & 0 & \cdots & -\frac{1-\gamma}{2} & C_A \\
	1 & 1 & 1 & \cdots & 1 & 1
	\end{array}
	\right],\qquad
	f = \left[\begin{array}{c}
	0\\
	0\\
	\vdots\\
	0\\
	1
	\end{array}
	\right].
\end{flalign*}
It can be checked that the solution of \cref{eq: 56} is strictly positive, i.e., if $F\alpha=f$, then we have $\alpha_l>0$ for $1\leq l \leq k$.
Recalling the definition of $\lalpha{\cdot}$-- norm, we have
\begin{flalign}
&\lalpha{M_\pi \Delta_G}\nonumber\\
&=\gamma_1\alpha_1\lsmu{[\msP_\pi\otimes\msI_{d_1}] \Delta_1} + \alpha_2 \lsmu{A_{2,1}\Delta_1 + \gamma_2 [\msP_\pi\otimes\msI_{d_2}] \Delta_2} \nonumber\\
&\quad + \cdots + \alpha_k\lsmu{\sum_{j=1}^{k-1} A_{k,j}\Delta_j +  \gamma_k [\msP_\pi\otimes\msI_{d_k}] \Delta_k}\nonumber\\
&\leq \gamma_1\alpha_1\lsmu{[\msP_\pi\otimes\msI_{d_1}] \Delta_1}  + \alpha_2\lsmu{A_{2,1}\Delta_1} + \cdots + \alpha_k \lsmu{A_{k,1}\Delta_1}\nonumber\\
&\quad + \gamma_2\alpha_2\lsmu{[\msP_\pi\otimes\msI_{d_2}] \Delta_2}  + \alpha_3\lsmu{A_{3,2}\Delta_2} + \cdots + \alpha_k \lsmu{A_{k,2}\Delta_2}\nonumber\\
&\quad + \cdots\nonumber\\
&\quad + \gamma_k \alpha_k \lsmu{[\msP_\pi\otimes\msI_{d_k}] \Delta_k}\nonumber\\
&\leq \left(\gamma\alpha_1 + C_A\sum_{i=2}^{k} \alpha_i\right)\lsmu{\Delta_1} + \left(\gamma\alpha_2 + C_A\sum_{i=3}^{k} \alpha_i\right)\lsmu{\Delta_2} + \cdots + \gamma\alpha_k \lsmu{\Delta_k}\nonumber\\
&\leq \frac{1+\gamma}{2}\alpha_1\lsmu{\Delta_1} + \frac{1+\gamma}{2}\alpha_2\lsmu{\Delta_2} + \cdots + \gamma\alpha_k \lsmu{\Delta_k}\nonumber\\
&\leq \frac{1+\gamma}{2}\left(\sum_{i=1}^{k}\alpha_i\lsmu{\Delta_i}\right)\nonumber\\
&=\frac{1+\gamma}{2}\lalpha{\Delta_G},\label{eq: 57}
\end{flalign}
where the first inequality follows from the triangle inequality, the second inequality follows from the fact that $A_{i,j}$ is bounded and \Cref{lemma13}, and the third inequality follows from the definition of $\gamma$ and the fact that $\alpha$ is the solution of \cref{eq: 56}. Obviously, \cref{eq: 57} implies the following property,
\begin{flalign*}
	\lalpha{\mct_\pi G^\prime_\pi - \mct_\pi G^{\prime\prime}_\pi}\leq \frac{1+\gamma}{2}\lalpha{G^\prime_\pi - G^{\prime\prime}_\pi},
\end{flalign*}
which completes the proof in the forward GBO evaluation setting.

We next consider the {\bf backward GBO setting}, where $\hat{\mct}_{G,\pi}$ is defined in \cref{eq: backward_prob}. Following steps similar to those from \cref{eq: 58} -- \cref{eq: 57}, we can obtain
\begin{flalign}
	&\lalpha{\hat{\mct}_\pi G_\pi - \hat{\mct}_\pi G^\prime_\pi} = \lalpha{M_\pi \Delta_G}\nonumber\\
	&=\gamma_1\alpha_1\lsmu{U^{-1}_{\pi,1}[\msP_\pi\otimes\msI_{d_1}]U_{\pi,1} \Delta_1} + \alpha_2 \lsmu{A_{2,1}\Delta_1 + \gamma_2 U^{-1}_{\pi,2}[\msP_\pi\otimes\msI_{d_2}]U_{\pi,2} \Delta_2} + \cdots \nonumber\\
	&\quad + \alpha_k\lsmu{\sum_{j=1}^{k-1} A_{k,j}\Delta_j +  \gamma_k U^{-1}_{\pi,k}[\msP_\pi\otimes\msI_{d_k}]U_{\pi,k} \Delta_k}\nonumber\\
	&\leq \gamma_1\alpha_1\lsmu{U^{-1}_{\pi,1}[\msP_\pi\otimes\msI_{d_1}]U_{\pi,1} \Delta_1}  + \alpha_2\lsmu{A_{2,1}\Delta_1} + \cdots + \alpha_k \lsmu{A_{k,1}\Delta_1}\nonumber\\
	&\quad + \gamma_2\alpha_2\lsmu{U^{-1}_{\pi,2}[\msP_\pi\otimes\msI_{d_2}]U_{\pi,2} \Delta_2}  + \alpha_3\lsmu{A_{3,2}\Delta_2} + \cdots + \alpha_k \lsmu{A_{k,2}\Delta_2}\nonumber\\
	&\quad + \cdots\nonumber\\
	&\quad + \gamma_k \alpha_k \lsmu{U^{-1}_{\pi,k}[\msP_\pi\otimes\msI_{d_k}]U_{\pi,k} \Delta_k}\nonumber\\
	&\leq \left(\gamma\alpha_1 + C\sum_{i=2}^{k} \alpha_i\right)\lsmu{\Delta_1} + \left(\gamma\alpha_2 + C\sum_{i=3}^{k} \alpha_i\right)\lsmu{\Delta_2} + \cdots + \gamma\alpha_k \lsmu{\Delta_k}\nonumber\\
	&\leq \frac{1+\gamma}{2}\alpha_1\lsmu{\Delta_1} + \frac{1+\gamma}{2}\alpha_2\lsmu{\Delta_2} + \cdots + \gamma\alpha_k \lsmu{\Delta_k}\nonumber\\
	&\leq \frac{1+\gamma}{2}\left(\sum_{i=1}^{k}\alpha_i\lsmu{\Delta_i}\right)\nonumber\\
	&=\frac{1+\gamma}{2}\lalpha{\Delta_G},\label{eq: 59}
\end{flalign}
where the first inequality follows from the triangle inequality, the second inequality follows from the fact that $A_{i,j}$ is bounded and \Cref{lemma14}, and the third inequality follows from the definition of $\gamma$ and the fact that $\alpha$ is the solution of \cref{eq: 56}. \Cref{eq: 59} implies the following
\begin{flalign*}
\lalpha{\hat{\mct}_\pi G^\prime_\pi - \hat{\mct}_\pi G^{\prime\prime}_\pi}\leq \frac{1+\gamma}{2}\lalpha{G_\pi^\prime - G^{\prime\prime}_\pi},
\end{flalign*}
which completes the proof in the backward GBO evaluation setting.

\subsection{Proof of \Cref{cond2}}
We first consider the {\bf forward GFV setting}. Recall the linear function approximation of $G_\pi$ is given by
\begin{flalign*}
G(\theta) = {\rm\Phi} \theta,
\end{flalign*}
where
\begin{flalign*}
{\rm\Phi} = \left[\begin{array}{cccc}
{\rm\Phi}_1\otimes\msI_{d_1} & 0 & \cdots &  0 \\
0 & {\rm\Phi}_2\otimes \msI_{d_2} & \cdots &  0 \\
\vdots & \vdots & & \vdots \\
0 & 0 & \cdots & {\rm\Phi}_k\otimes \msI_{d_k}
\end{array}
\right],\quad
\theta = \left[\begin{array}{c}
\mvec(\theta_1^\top)\\
\mvec(\theta^\top_2)\\
\vdots\\
\mvec(\theta^\top_k)
\end{array}
\right].
\end{flalign*}
Folloing the definition of $g(\theta)$ in \cref{eq: 48}, we have
\begin{flalign*}
	-g(\theta) &= {\rm\Phi}^\top U_\pi (\mct_{G,\pi}G(\theta) - G(\theta))\nonumber\\
	& = {\rm\Phi}^\top U_\pi ((M_\pi - I){\rm\Phi} \theta + B)\nonumber\\
	& = G\theta + g,
\end{flalign*}
where $G = {\rm\Phi}^\top U_\pi (M_\pi - I){\rm\Phi}$ and $g = {\rm\Phi}^\top U_\pi B$. Since the monotonicity depends only on the matrix $G$, we next proceed to show that $G$ is Hurwitz. For the matrix $G$, we have
\begin{flalign}
	G & = {\rm\Phi}^\top U_\pi (M_\pi - I){\rm\Phi}  = \left[\begin{array}{cccc}
	A_1 & 0 & \cdots & 0 \\
	N_{2,1} & A_2 & \cdots & 0\\
	\vdots & \vdots \\
	N_{k,1} & N_{k,2} & \cdots & A_k
	\end{array}
	\right],\label{eq: 61}
\end{flalign}
where $A_i = [{\rm\Phi}_i^\top U_{\pi,i} (\gamma_i\msP_\pi - I){\rm\Phi}_i] \otimes \msI_{d_i}$ and $N_{i,j}$ is a matrix that depends on ${\rm\Phi}_i$, ${\rm\Phi}_j$, $\msP_\pi$ and $\mu_\pi$. We have the following equations hold:
\begin{flalign}
	\text{eig}(G) &= \{\text{eig}(A_1),\cdots,\text{eig}(A_k) \},\label{eq: 63}\\
	\text{eig}(A_i) &= \text{eig}({\rm\Phi}_i^\top U_{\pi,i} (\gamma_i\msP_\pi - I){\rm\Phi}_i),\label{eq: 64}\\
	\max\{\text{eig}({\rm\Phi}_i^\top U_{\pi,i} (\gamma_i\msP_\pi - I){\rm\Phi}_i)\} &=-(1-\gamma)\zeta_i, \label{eq: 65}
\end{flalign}
where $\zeta_i$ is defined in \Cref{cond2}, the first equation follows because the eigenvalue of a matrix is determined by the eigenvalues of its diagonal block matrices \cite{horn2012matrix}, the second equation follows from the fact that $\text{eig}(M\otimes \msI_d) = \text{eig}(M)$ for any matrix $M$ and positive integer $d$, and the last follows from \cite[Lemma 1, Lemma 3]{bhandari2018finite}. Combining \cref{eq: 63}--(\ref{eq: 65}), we can obtain equation
\begin{flalign}
	\max\{\text{eig}(G)\} \leq -(1-\gamma)\min_i \zeta_i = -\lambda_G < 0,\label{eq: 66}
\end{flalign}
which completes the proof in the forward GVF setting.

We next consider the {\bf backward GVF setting}. Following the steps similar to those for deriving \cref{eq: 61}, we can obtain $-g(\theta) = \hat{G}\theta + \hat{g}$, where $\hat{g}={\rm\Phi}^\top P^\top_\pi U_\pi B$. For the matrix $\hat{G}$, we have
\begin{flalign}
	\hat{G} = {\rm\Phi}^\top P_\pi^\top (\hat{M}_\pi - I) U_\pi {\rm\Phi} = \left[\begin{array}{cccc}
	\hat{A}_1 & 0 & \cdots & 0 \\
	\hat{N}_{2,1} & \hat{A}_2 & \cdots & 0\\
	\vdots & \vdots \\
	\hat{N}_{k,1} & \hat{N}_{k,2} & \cdots & \hat{A}_k
	\end{array}
	\right],\label{eq: 62}
\end{flalign}
where $\hat{A}_{i} = [{\rm\Phi}_i^\top (\gamma_i\msP^\top_\pi - I)U_{\pi,i}{\rm\Phi}_i] \otimes \msI_{d_i}$ and $\hat{N}_{i,j}$ is a matrix that depends on ${\rm\Phi}_i$, ${\rm\Phi}_j$, $\msP_\pi$ and $\mu_\pi$. Following the steps similar to those in \cref{eq: 63}--(\ref{eq: 66}) and using the result in the verification of \cite[item (c) in Assumption 2]{zhang2020learning}, we have
\begin{flalign*}
\max\{\text{eig}(G)\}\leq -(1-\gamma)\min_i \zeta_i = -\lambda_G < 0,
\end{flalign*}
which completes the proof in the backward GVF setting.

\section{Proof of \Cref{thm1} }

\subsection{Supporting Lemmas}
We first develop the property for the update of $w_\rho$ in \Cref{algorithm_offtd}. Given a sample $(s_t,a_t,B_t,s^\prime_t)\sim \mcd$ and $a^\prime_t\sim \pi(\cdot|s^\prime_t)$, we  introduce the following definitions.
\begin{flalign}
P_{t} &= \left[\begin{array}{ccc}
\psi_t^\top\psi_t & (\psi_t-\psi^\prime_t)\psi^\top_t & 0 \\
-\psi_t(\psi^\top_t-\psi^{\prime\top}_t) & 0 & \psi_t \\
0 & -\psi^\top_t & 1
\end{array}
\right],\qquad
p_{t} = \left[\begin{array}{c}
0\\
0\\
1
\end{array}
\right].\nonumber
\end{flalign}
Consider the matrix $P_t$ and vector $p_t$, we have the following holds
\begin{flalign}
	\lF{P_t}^2&= \lF{\psi_t^\top\psi_t }^2 + 2\lF{(\psi_t-\psi^\prime_t)\psi^\top_t}^2 + 2\lF{\psi_t}^2 + 1\nonumber\\
	&\leq 9 C^4_\psi + 2C^2_\psi + 1,
\end{flalign}
where $C_\psi$ is the upper bound on the feature fector $\psi(\cdot)$, i.e., $\ltwo{\psi(s,a)}\leq C_\psi$ for all $(s,a)\in\mcs\times\mca$,
which implies $\lF{P_t}\leq C_P$, where
\begin{flalign}
	C_P=\sqrt{9 C^4_\psi + 2C^2_\psi + 1}.\label{eq: 7}
\end{flalign}
For the vector $p_t$, it can be checked easily that $\ltwo{p_t}\leq 1$.

We also define $P=\mE_{\mcd\cdot\pi}[P_t]$ and $p = \mE_{\mcd\cdot\pi}[p_t]$, i.e.,
\begin{flalign}
P&= \left[\begin{array}{ccc}
\mE_{\mcd\cdot\pi}[\psi^\top\psi] & \mE_{\mcd\cdot\pi}[(\psi-\psi^\prime)\psi^\top] & 0 \\
-\mE_{\mcd\cdot\pi}[\psi(\psi^\top-\psi^{\prime\top})] & 0 & \mE_{\mcd\cdot\pi}[\psi] \\
0 & -\mE_{\mcd\cdot\pi}[\psi^\top] & 1
\end{array}
\right],\qquad
p = \left[\begin{array}{c}
0\\
0\\
1
\end{array}
\right].\nonumber
\end{flalign}
Note that
\begin{flalign*}
	\left[\begin{array}{c}
	\nabla_{w_f} L(\hat{\rho}, \hat{f}, \eta)\\
	\nabla_{w_\rho} L(\hat{\rho}, \hat{f}, \eta)\\
	\nabla_{\eta} L(\hat{\rho}, \hat{f}, \eta)
	\end{array}
	\right] = -(P\kappa + p).
\end{flalign*}
It has been shown in \cite[Theorem 2]{zhang2020gradientdice} that the real parts of all eigenvalues of $P$ are strictly positive, which guarantees that there exists a positive constant $\lambda_P$ such that 
\begin{flalign}
    \langle Px,x\rangle\ge \lambda_P\ltwo{x}^2\quad \text{for all} \quad x\in \mR^{2d_\rho}.\label{eq: 71}
\end{flalign}
We also define $\kappa_t = [w_{f,t}^\top,w_{\rho,t}^\top,\eta_t]^\top$. The update of density ratio learning can be rewritten as
\begin{flalign*}
	\kappa_{t+1}=\kappa_t - \beta_t \zeta(x_t,\kappa_t),
\end{flalign*}
where $\zeta(x_t,\kappa_t) = P_t\kappa_t + p_t$. We also define the population update as $\zeta(\kappa_t) = \mE_{D\cdot\pi}[\zeta(x,\kappa_t)] = P\kappa_t + p$. Without loss of generality, we assume that there exists a positive constant $C_\kappa$ such that $\ltwo{\kappa^*}\leq C_\kappa$, where $\kappa^*$ is the global optimum of the density ratio learning defined as
\begin{flalign}
\langle \zeta(\kappa^*), \kappa - \kappa^* \rangle \leq 0,\quad \forall \kappa\in \mR^{d_\rho} \times R_\rho \times \mR.
\end{flalign}

The following lemma, often referred to as the "three-points" lemma, characterizes the incremental updating progress of $\kappa_t$ with projection, a proof of which can be found in \cite[Lemma 3.1]{lan2020first}.
\begin{lemma}\label{lemma1}
	Consider the update of $w_{f,t}$, $w_{{\rho,t}}$ and $\eta_t$ in \Cref{algorithm_offtd}. For all $\kappa\in \mR^M \times R_\rho \times \mR$, we have the following holds
	\begin{flalign}
		\beta_t\langle \zeta(x_t,\kappa_t), \kappa_{t+1} - \kappa\rangle + \frac{1}{2}\ltwo{\kappa_{t+1}-\kappa_t}^2\leq \frac{1}{2}\ltwo{\kappa_t-\kappa}^2 - \frac{1}{2}\ltwo{\kappa_{t+1}-\kappa}^2.\label{eq: 1}
	\end{flalign}
\end{lemma}

Similarly to \Cref{lemma1}, we also have the following "three-points lemma" for the iteration of $\theta_t$.
\begin{lemma}\label{lemma6}
	Consider the update of $\theta_t$ in \Cref{algorithm_offtd}. For all $\theta\in  R_\theta$, we have the following holds
	\begin{flalign}
	-\alpha_t\langle \hat{\rho}(x_t,w_{\rho,t})g(x_t,\theta_t), \theta_{t+1} - \theta \rangle + \frac{1}{2}\ltwo{\theta_{t+1}-\theta_t}^2\leq \frac{1}{2}\ltwo{\theta_t-\theta}^2 - \frac{1}{2}\ltwo{\theta_{t+1}-\theta}^2,\label{eq: 11}
	\end{flalign}
	where $\hat{\rho}(x_t,w_{\rho,t})$ is defined in \cref{eq: 49}.
\end{lemma}

The following lemma characterizes the smoothness of $\zeta(\cdot)$.
\begin{lemma}\label{lemma3}
	For any $\kappa,\kappa^\prime\in \mR^{d_\rho} \times R_\rho \times \mR$, we have
	\begin{flalign*}
		\ltwo{\zeta(\kappa) - \zeta(\kappa^\prime)}\leq C_P\ltwo{\kappa-\kappa^\prime},
	\end{flalign*}
	where $C_P$ is defined in \cref{eq: 7}.
\end{lemma}
\begin{proof}
	Recalling the definition of $\zeta(\kappa) = P\kappa + p$, we can obtain the following
	\begin{flalign*}
		\ltwo{\zeta(\kappa) - \zeta(\kappa^\prime)} = \ltwo{P(\kappa-\kappa^\prime)}\leq \ltwo{P}\ltwo{\kappa-\kappa^\prime}\leq C_P\ltwo{\kappa-\kappa^\prime},
	\end{flalign*}
	which completes the proof.
\end{proof}

Similarly, the following lemma characterizes the smoothness of $g(\theta) = \mE_{\mu_\pi}[g(x,\theta)]$.
\begin{lemma}\label{lemma7}
	In both the forward and backward GVF evaluation settings, for any $\theta,\theta^\prime\in\mR^{d_\rho}$, we have
	\begin{flalign*}
		\ltwo{g(\theta)-g(\theta^\prime)}\leq C_g\lF{\theta-\theta^\prime},
	\end{flalign*}
	where $C_g=(d_gC_\phi C_m+1)C_\phi$.
\end{lemma}
\begin{proof}
	First consider the forward GVF evaluation setting. Recall the definition of $g(\theta)$ and $x=(s,a,s^\prime,a^\prime)$, we have
	\begin{flalign*}
		g(\theta) = \mE_{\mu_\pi}[\phi(s,a)^\top (B(x)+m(x)\phi(s^\prime,a^\prime)\theta - \phi(s,a)\theta)],
	\end{flalign*}
	which implies
	\begin{flalign}
		\ltwo{g(\theta) - g(\theta^\prime)} 
		&= \ltwo{\mE_{\mu_\pi}[\phi(s,a)(m(x)\phi(s^\prime,a^\prime)(\theta-\theta^\prime) + \phi(s,a)(\theta^\prime-\theta))]}\nonumber\\
		&\leq \mE_{\mu_\pi}[(\lF{\phi(s,a)}\lF{m(x)} + 1) \ltwo{\theta^\prime-\theta} \lF{\phi(s,a)} ] \nonumber\\
		&\leq C_g\ltwo{\theta-\theta^\prime}.\label{eq: 16}
	\end{flalign}
	Following the steps similar to those in \cref{eq: 16}, we can also prove that $\ltwo{g(\theta)-g(\theta^\prime)}\leq C_g\lF{\theta-\theta^\prime}$ holds in the backward GVF evaluation setting.
\end{proof}

The following lemma characterizes the monotonicity of $\zeta(\cdot)$.
\begin{lemma}\label{lemma4}
	We have the following holds
	\begin{flalign*}
		\langle \zeta(\kappa),\kappa - \kappa^* \rangle \geq \lambda_P\ltwo{\kappa-\kappa^*}^2,\quad \forall \kappa\in \mR^{d_\rho} \times R_\rho \times \mR.
	\end{flalign*}
\end{lemma}
\begin{proof}
	Recall that $P$ is strictly positive defined (\cref{eq: 71}). We have
	\begin{flalign}
			\langle \zeta(\kappa),\kappa - \kappa^* \rangle &= 	\langle \zeta(\kappa^*),\kappa - \kappa^* \rangle + 	\langle \zeta(\kappa) - \zeta(\kappa^*),\kappa - \kappa^* \rangle\nonumber\\
			& \geq \langle \zeta(\kappa) - \zeta(\kappa^*),\kappa - \kappa^* \rangle \nonumber\\
			&= \langle P(\kappa-\kappa^*), \kappa - \kappa^* \rangle \nonumber\\
			&\geq \lambda_P\ltwo{\kappa-\kappa^*}^2,
	\end{flalign}
	which completes the proof.
\end{proof}

The next lemma bounds the per-iteration variance of the update of $\kappa_t$.
\begin{lemma}\label{lemma5}
	Given a sample $(s_t,a_t,B_t,s^\prime_t)\sim \mcd_d$ and $a^\prime_t\sim \pi(\cdot|s^\prime_t)$ and any $\kappa\in \mR^{d_\rho} \times R_\rho \times \mR$, we have the following holds
	\begin{flalign*}
		\ltwo{\zeta(x_t,\kappa) - \zeta(\kappa)}^2\leq 8C^2_P \ltwo{\kappa-\kappa^*}^2 + 8C^2_PC^2_\kappa.
	\end{flalign*}
\end{lemma}
\begin{proof}
	Recalling the definitions of $\zeta(x_t,\kappa)=P_t\kappa + p_t$ and $\zeta(\kappa)=P\kappa + p$, we can obtain the following
	\begin{flalign*}
		\ltwo{\zeta(x_t,\kappa) - \zeta(\kappa)}^2 & = \ltwo{(P_t-P)\kappa}^2 =  2\ltwo{(P_t-P)(\kappa-\kappa^*)}^2 + 2\ltwo{(P_t-P)\kappa^*}^2\nonumber\\
		&\leq 8C^2_P\ltwo{\kappa-\kappa^*}^2 + 8C^2_PC^2_\kappa.
	\end{flalign*}
\end{proof}

The following lemma bounds the norm of the stochastic update $g(x,\theta)$ and the per-iteration variance of GenTD update with density ratio $\rho(s,a)$.
\begin{lemma}\label{lemma8}
	Given a sample $(s_t,a_t,B_t,s^\prime_t)\sim \mcd$ and $a^\prime_t\sim \pi(\cdot|s^\prime_t)$ and any $\theta\in R_\theta$, we have the following holds
	\begin{flalign*}
	\ltwo{g(x_t,\theta)}\leq D_g,\quad \text{and}\quad \mE[\ltwo{\rho(s_t,a_t)g(x_t,\theta) - g(\theta)}^2]\leq V_g,
	\end{flalign*}
	where $D_g=d_g C_\phi [C_{\max} + (C_m+1)D_\theta C_\phi]$ and $V_g=2\rho_{\max}D_g$.
\end{lemma}
\begin{proof}
	We prove the first result as follows,
	\begin{flalign*}
		\ltwo{g(x_t,\theta)} &= \ltwo{\phi(s,a)^\top(B(x)+m(s^\prime,a^\prime)\phi(s^\prime,a^\prime)\theta - \phi(s,a)\theta)}\nonumber\\
		&\leq \lF{\phi(s,a)^\top B(x)} + \lF{\phi(s,a)}\lF{m(s^\prime,a^\prime)\phi(s^\prime,a^\prime)\theta - \phi(s,a)\theta}\nonumber\\
		&\leq d_gC_\phi [C_{\max} + (C_m+1)D_\theta C_\phi],
	\end{flalign*}
	where the last inequality follows from the boundness of the set $R_\theta$. Here we consider $\ltwo{\theta}\leq D_\theta$ for all $\theta\in R_\theta$.
	The second result can be obtained as follows
	\begin{flalign}
		\ltwo{\rho(s_t,a_t)g(x_t,\theta) - g(\theta)}^2\leq \lone{\rho(s_t,a_t)}(\ltwo{g(x_t,\theta)} + \ltwo{g(\theta)})\leq 2\rho_{\max}D_g.
	\end{flalign}
\end{proof}

We next bound the convergence rate of $w_{\rho,t}$.
\begin{lemma}\label{lemma2}
	Consider $w_{f,t}$, $w_{{\rho,t}}$ and $\eta_t$ in \Cref{algorithm_offtd}. Let stepsize $\beta_t = \frac{2}{\lambda_P(t+t_0+1)}$ where $t_0=\frac{36C^2_P}{\lambda^2_P}$. For any $\kappa\in  \mR^{d_\rho} \times R_\rho \times \mR$, we have
	\begin{flalign*}
		\mE[\ltwo{\kappa_T-\kappa^*}^2]\leq \frac{(1+16\beta^2_0C^2_P)(t_0+1)^2\ltwo{\kappa_0-\kappa^*}^2}{(T+t_0-1)(T+t_0)} + \frac{64C^2_PC^2_\kappa}{(T+t_0)\lambda^2_P}.
	\end{flalign*}
\end{lemma}
\begin{proof}
	The inner product in \cref{eq: 1} can be equivalently written as
	\begin{flalign}
		&\langle \zeta(x_t,\kappa_t), \kappa_{t+1} - \kappa\rangle \nonumber\\
		&\quad= \langle \zeta(\kappa_{t+1}), \kappa_{t+1} - \kappa\rangle + \langle \zeta(\kappa_{t}) - \zeta(\kappa_{t+1}), \kappa_{t+1} - \kappa\rangle + \langle \zeta(x_t,\kappa_t)-\zeta(\kappa_{t}), \kappa_{t} - \kappa\rangle\nonumber\\
		&\quad\quad  + \langle \zeta(x_t,\kappa_t)-\zeta(\kappa_{t}), \kappa_{t+1} - \kappa_t\rangle\nonumber\\
		&\quad\geq \langle \zeta(\kappa_{t+1}), \kappa_{t+1} - \kappa\rangle -  \ltwo{\zeta(\kappa_{t}) - \zeta(\kappa_{t+1})} \ltwo{\kappa_{t+1} - \kappa} + \langle \zeta(x_t,\kappa_t)-\zeta(\kappa_{t}), \kappa_{t} - \kappa\rangle\nonumber\\
		&\quad\quad  -  \ltwo{\zeta(x_t,\kappa_t)-\zeta(\kappa_{t})} \ltwo{\kappa_{t+1} - \kappa_t} \nonumber\\
		&\quad\geq \langle \zeta(\kappa_{t+1}), \kappa_{t+1} - \kappa\rangle -  C_P\ltwo{\kappa_{t} - \kappa_{t+1}} \ltwo{\kappa_{t+1} - \kappa} + \langle \zeta(x_t,\kappa_t)-\zeta(\kappa_{t}), \kappa_{t} - \kappa\rangle\nonumber\\
		&\quad\quad  -  \ltwo{\zeta(x_t,\kappa_t)-\zeta(\kappa_{t})} \ltwo{\kappa_{t+1} - \kappa_t},\label{eq: 2}
	\end{flalign}
	where the last inequality follows from \Cref{lemma3}. Substituting \cref{eq: 2} into \cref{eq: 1}, we obtain
	\begin{flalign}
		&\frac{1}{2}\ltwo{\kappa_t-\kappa}^2 - \frac{1}{2}\ltwo{\kappa_{t+1}-\kappa}^2 \nonumber\\
		&\quad \geq \beta_t \langle \zeta(\kappa_{t+1}), \kappa_{t+1} - \kappa\rangle -  \beta_tC_P\ltwo{\kappa_{t} - \kappa_{t+1}} \ltwo{\kappa_{t+1} - \kappa} + \beta_t \langle \zeta(x_t,\kappa_t)-\zeta(\kappa_{t}), \kappa_{t} - \kappa\rangle\nonumber\\
		&\quad\quad  - \beta_t \ltwo{\zeta(x_t,\kappa_t)-\zeta(\kappa_{t})} \ltwo{\kappa_{t+1} - \kappa_t} + \frac{1}{2}\ltwo{\kappa_{t+1}-\kappa_t}^2.\label{eq: 3}
	\end{flalign}
	Note that we have the following holds
	\begin{flalign}
		&\frac{1}{2}\ltwo{\kappa_{t+1}-\kappa_t}^2 - \beta_tC_P\ltwo{\kappa_{t} - \kappa_{t+1}} \ltwo{\kappa_{t+1} - \kappa} - \beta_t \ltwo{\zeta(x_t,\kappa_t)-\zeta(\kappa_{t})} \ltwo{\kappa_{t+1} - \kappa_t}\nonumber\\
		&= \frac{1}{4}\ltwo{\kappa_{t+1}-\kappa_t}^2 - \beta_tC_P\ltwo{\kappa_{t} - \kappa_{t+1}} \ltwo{\kappa_{t+1} - \kappa} + \frac{1}{4}\ltwo{\kappa_{t+1}-\kappa_t}^2 \nonumber\\
		&\quad - \beta_t \ltwo{\zeta(x_t,\kappa_t)-\zeta(\kappa_{t})} \ltwo{\kappa_{t+1} - \kappa_t} \nonumber\\
		&\geq - \beta^2_t C^2_P  \ltwo{\kappa_{t+1} - \kappa}^2 - \beta^2_t \ltwo{\zeta(x_t,\kappa_t)-\zeta(\kappa_{t})}^2.\label{eq: 4}
	\end{flalign}
	Substituting \cref{eq: 4} in \cref{eq: 3} yields
	\begin{flalign}
		&\frac{1}{2}\ltwo{\kappa_t-\kappa}^2 - \frac{1}{2}\ltwo{\kappa_{t+1}-\kappa}^2 \nonumber\\
		&\geq \beta_t \langle \zeta(\kappa_{t+1}), \kappa_{t+1} - \kappa\rangle + \beta_t \langle \zeta(x_t,\kappa_t)-\zeta(\kappa_{t}), \kappa_{t} - \kappa\rangle - \beta^2_t C^2_P  \ltwo{\kappa_{t+1} - \kappa}^2 \nonumber\\
		&\quad - \beta^2_t \ltwo{\zeta(x_t,\kappa_t)-\zeta(\kappa_{t})}^2.\label{eq: 5}
	\end{flalign}
	Rearranging \cref{eq: 5} and letting $\kappa = \kappa^*$ yield
	\begin{flalign}
		&\ltwo{\kappa_t-\kappa^*}^2 + 2\beta^2_t \ltwo{\zeta(x_t,\kappa_t)-\zeta(\kappa_{t})}^2 \nonumber\\
		&\geq (1-2\beta^2_t C^2_P) \ltwo{\kappa_{t+1}-\kappa^*}^2 + 2\beta_t \langle \zeta(\kappa_{t+1}), \kappa_{t+1} - \kappa^*\rangle + 2\beta_t \langle \zeta(x_t,\kappa_t)-\zeta(\kappa_{t}), \kappa_{t} - \kappa^*\rangle\nonumber\\
		&\geq (1 + 2\beta_t\lambda_P -2\beta^2_t C^2_P)\ltwo{\kappa_{t+1}-\kappa^*}^2 + 2\beta_t \langle \zeta(x_t,\kappa_t)-\zeta(\kappa_{t}), \kappa_{t} - \kappa^*\rangle,\label{eq: 6}
	\end{flalign}
	where the last inequality follows from $\langle \zeta(\kappa_{t+1}), \kappa_{t+1} - \kappa^*\rangle\leq \lambda_P\ltwo{\kappa_{t+1}-\kappa^*}^2$. Taking expectation on both sides of \cref{eq: 6}, and noting that $\mE[ \langle \zeta(x_t,\kappa_t)-\zeta(\kappa_{t}), \kappa_{t} - \kappa^*\rangle|\mf_t]=0$, we obtain
	\begin{flalign}
		&(1 + 2\beta_t\lambda_P -2\beta^2_t C^2_P) \mE[\ltwo{\kappa_{t+1}-\kappa^*}^2]\nonumber\\
		&\quad \leq \mE[\ltwo{\kappa_t-\kappa^*}^2] + 2\beta^2_t \mE[\ltwo{\zeta(x_t,\kappa_t)-\zeta(\kappa_{t})}^2]\nonumber\\
		&\quad \leq (1+16\beta_t^2C^2_P)\mE[\ltwo{\kappa_t-\kappa^*}^2] + 16\beta^2_tC^2_PC^2_\kappa.\label{eq: 9}
	\end{flalign}
	Multiplying both sides of \cref{eq: 9} with $I_t$ and summing over $t=0,\cdots,T-1$ yield
	\begin{flalign}
		\sum_{t=0}^{T-1}a_t\mE[\ltwo{\kappa_{t+1}-\kappa^*}^2] \leq \sum_{t=0}^{T-1}b_t\mE[\ltwo{\kappa_{t}-\kappa^*}^2] + c,\label{eq: 10}
	\end{flalign}
	where
	\begin{flalign*}
		a_t &= (1 + 2\beta_t\lambda_P -2\beta^2_t C^2_P)I_t,\nonumber\\
		b_t &= (1+16\beta_t^2C^2_P)I_t,\nonumber\\
		c &= 16C^2_PC^2_\kappa\sum_{t=0}^{T-1}\beta^2_t I_t.
	\end{flalign*}
	We further let
	\begin{flalign*}
		I_t &= (t+t_0)(t+t_0+1),\nonumber\\
		\beta_t &= \frac{2}{\lambda_P(t+t_0-1)},\nonumber\\
		t_0 &= \frac{36C^2_P}{\lambda^2_P}+1.
	\end{flalign*}
	We can obtain the following
	\begin{flalign*}
		a_t - b_{t+1} &= (1 + 2\beta_t\lambda_P -2\beta^2_t C^2_P)I_t - (1+16\beta_{t+1}^2C^2_P)s_{t+1}\nonumber\\
		&\geq (1 + 2\beta_t\lambda_P)I_t - (1+ 2\beta^2_t +16\beta_{t+1}^2C^2_P)s_{t+1}\nonumber\\
		&\geq (1 + 2\beta_t\lambda_P)I_t - (1+ 18\beta_{t}^2C^2_P)s_{t+1}\nonumber\\
		&\geq (1 + 2\beta_t\lambda_P)I_t - (1 + \beta_t\lambda_P)s_{t+1}\nonumber\\
		&\geq (t+t_0+1) \frac{(t+t_0+3)(t+t_0) - (t+t_0+2)^2}{t+t_0-1} \nonumber\\
		&\geq 0.
	\end{flalign*}
	Substituting the above results into \cref{eq: 10} yields
	\begin{flalign*}
		a_{T-1}t\mE[\ltwo{\kappa_T-\kappa^*}^2]\leq b_0\ltwo{\kappa_0-\kappa^*}^2 + c,
	\end{flalign*}
	which implies
	\begin{flalign*}
		\mE[\ltwo{\kappa_T-\kappa^*}^2] &\leq \frac{b_0\ltwo{\kappa_0-\kappa^*}^2}{A_{T-1}} + \frac{c}{a_{T-1}}\nonumber\\
		&=\frac{(1+16\beta^2_0C^2_P)s_0\ltwo{\kappa_0-\kappa^*}^2}{(1 + 2\beta_{T-1}\lambda_P -2\beta^2_{T-1} C^2_P)s_{T-1}} + \frac{16C^2_PC^2_\kappa\sum_{t=0}^{T-1}\beta^2_t I_t}{(1 + 2\beta_{T-1}\lambda_P -2\beta^2_{T-1} C^2_P)s_{T-1}}\nonumber\\
		&\leq \frac{(1+16\beta^2_0C^2_P)(t_0+1)^2\ltwo{\kappa_0-\kappa^*}^2}{(T+t_0-1)(T+t_0)} + \frac{64C^2_PC^2_\kappa}{(T+t_0)\lambda^2_P},
	\end{flalign*}
	which completes the proof.
\end{proof}

Note that \Cref{lemma2} implies that there exists a positive number $D_\rho$ such that
\begin{flalign}
\mE[\ltwo{w_{\rho,t} - w^*_\rho}^2]\leq \frac{D_\rho}{t+t_0}.\label{eq: 17}
\end{flalign}

\subsection{Proof of \Cref{thm1}}

Consider the inner product term in \cref{eq: 11}. We have
\begin{flalign}
	-\langle& \hat{\rho}(x_t,w_{\rho,t})g(x_t,\theta_t), \theta_{t+1} - \theta)\rangle\nonumber\\
	&=-\langle g(\theta_{t+1}), \theta_{t+1} - \theta)\rangle - \langle g(\theta_{t} - g(\theta_{t+1}), \theta_{t+1} - \theta)\rangle \nonumber\\
	&\quad -  \langle \rho(x_t)g(x_t,\theta_{t}) - g(\theta_t), \theta_{t+1} - \theta_t\rangle -  \langle \rho(x_t)g(x_t,\theta_{t}) - g(\theta_t), \theta_{t} - \theta\rangle \nonumber\\
	&\quad -  \langle( \hat{\rho}(x_t,w^*_\rho) - \rho(x_t))g(x_t,\theta_{t}), \theta_{t+1} - \theta\rangle \nonumber\\
	&\quad - \langle( \hat{\rho}(x_t,w_{\rho,t})
	-\hat{\rho}(x_t,w^*_\rho))g(x_t,\theta_{t}), \theta^\top_{t+1} - \theta\rangle\nonumber\\
	&\geq -\langle g(\theta_{t+1}), \theta^\top_{t+1} - \theta\rangle -  C_g\ltwo{\theta_{t} - \theta_{t+1}} \ltwo{\theta_{t+1} - \theta} \nonumber\\
	&\quad -   \ltwo{\rho(x_t)g(x_t,\theta_{t}) - g(\theta_t)} \ltwo{\theta_{t+1} - \theta_t}  -  \langle \rho(x_t)g(x_t,\theta_{t}) - g(\theta_t), \theta^\top_{t} - \theta\rangle \nonumber\\
	&\quad -   \lone{\hat{\rho}(x_t,w^*_\rho) - \rho(x_t)}  \ltwo{g(x_t,\theta_{t})} \ltwo{\theta_{t+1} - \theta} \nonumber\\
	&\quad -  \lone{\hat{\rho}(x_t,w_{\rho,t})
	-\hat{\rho}(x_t,w^*_\rho)}  \ltwo{g(x_t,\theta_{t})}  \ltwo{\theta_{t+1} - \theta},\label{eq: 12}
\end{flalign}
where the last inequality follows from \Cref{lemma7}. Substituting \cref{eq: 12} into \cref{eq: 11} yields
\begin{flalign}
	&\frac{1}{2}\ltwo{\theta_t-\theta}^2 - \frac{1}{2}\ltwo{\theta_{t+1}-\theta}^2\nonumber\\
	&\quad\geq -\alpha_t\langle g(\theta_{t+1}), \theta^\top_{t+1} - \theta \rangle -  \alpha_t C_g\ltwo{\theta_{t} - \theta_{t+1}} \ltwo{\theta_{t+1} - \theta} \nonumber\\
	&\quad\quad -  \alpha_t \ltwo{\rho(x_t)g(x_t,\theta_{t}) - g(\theta_t)} \ltwo{\theta_{t+1} - \theta_t}  -  \alpha_t\langle \rho(x_t)g(x_t,\theta_{t}) - g(\theta_t), \theta^\top_{t} - \theta^\top\rangle \nonumber\\
	&\quad\quad - \alpha_t  \lone{\hat{\rho}(x_t,w^*_\rho) - \rho(x_t)}  \ltwo{g(x_t,\theta_{t})} \ltwo{\theta_{t+1} - \theta} \nonumber\\
	&\quad\quad - \alpha_t \lone{\hat{\rho}(x_t,w_{\rho,t})
		-\hat{\rho}(x_t,w^*_\rho)}  \ltwo{g(x_t,\theta_{t})}  \ltwo{\theta_{t+1} - \theta} + \frac{1}{2}\ltwo{\theta_{t+1}-\theta_t}^2.\label{eq: 13}
\end{flalign}
We have the following holds
\begin{flalign}
	&\frac{1}{2}\ltwo{\theta_{t+1}-\theta_t}^2 -  \alpha_tC_g\ltwo{\theta_{t} - \theta_{t+1}} \ltwo{\theta_{t+1} - \theta} -   \alpha_t\ltwo{\rho(x_t)g(x_t,\theta_{t}) - g(\theta_t)} \ltwo{\theta_{t+1} - \theta_t}\nonumber\\
	& \quad\quad = \frac{1}{4}\ltwo{\theta_{t+1}-\theta_t}^2 -  \alpha_tC_g\ltwo{\theta_{t} - \theta_{t+1}} \ltwo{\theta_{t+1} - \theta} + \frac{1}{4}\ltwo{\theta_{t+1}-\theta_t}^2 \nonumber\\
	&\quad\quad \quad - \alpha_t  \ltwo{\rho(x_t)g(x_t,\theta_{t}) - g(\theta_t)} \ltwo{\theta_{t+1} - \theta_t}\nonumber\\
	&\quad\quad\geq -\alpha^2_t C^2_g\ltwo{\theta_{t+1}-\theta}^2 - \alpha^2_t\ltwo{\rho(x_t)g(x_t,\theta_{t}) - g(\theta_t)}^2,
\end{flalign}
which implies
\begin{flalign}
	&\frac{1}{2}\ltwo{\theta_t-\theta}^2 - \frac{1}{2}\ltwo{\theta_{t+1}-\theta}^2\nonumber\\
	&\quad\quad\geq -\alpha_t\langle g(\theta_{t+1}), \theta_{t+1} - \theta \rangle -\alpha^2_t C^2_g\ltwo{\theta_{t+1}-\theta}^2 - \alpha^2_t\ltwo{\rho(x_t)g(x_t,\theta_{t}) - g(\theta_t)}^2 \nonumber\\
	&\quad\quad  -  \alpha_t\langle \rho(x_t)g(x_t,\theta_{t}) - g(\theta_t), \theta_{t} - \theta \rangle - \alpha_t D_g \lone{\hat{\rho}(x_t,w^*_\rho) - \rho(x_t)}  \ltwo{\theta_{t+1} - \theta} \nonumber\\
	&\quad\quad - \alpha_t D_g \lone{\hat{\rho}(x_t,w_{\rho,t})
		-\hat{\rho}(x_t,w^*_\rho)}   \ltwo{\theta_{t+1} - \theta} \label{eq: 14},
\end{flalign}
where we use the fact that $\ltwo{g(x_t,\theta_t)}\leq D_g$ in \Cref{lemma8}. Rearranging \cref{eq: 14} and letting $\theta=\theta^*$ yield
\begin{flalign}
	&\ltwo{\theta_t-\theta^*}^2 + 2\alpha^2_t\ltwo{\rho(x_t)g(x_t,\theta_{t}) - g(\theta_t)}^2 \nonumber\\
	&\quad\geq \ltwo{\theta_{t+1}-\theta^*}^2 - 2\alpha_t\langle g(\theta_{t+1}), \theta_{t+1} - \theta^* \rangle -2\alpha^2_t C^2_g\ltwo{\theta_{t+1}-\theta^*}^2 \nonumber\\
	&\quad\quad  -  2\alpha_t\langle \rho(x_t)g(x_t,\theta_{t}) - g(\theta_t), \theta_{t} - \theta^* \rangle -2 \alpha_t D_g \lone{\hat{\rho}(x_t,w^*_\rho) - \rho(x_t)}  \ltwo{\theta_{t+1} - \theta^*} \nonumber\\
	&\quad\quad -2 \alpha_t D_g \lone{\hat{\rho}(x_t,w_{\rho,t})
		-\hat{\rho}(x_t,w^*_\rho)}   \ltwo{\theta_{t+1} - \theta}\nonumber\\
	&\quad\geq (1+2\alpha_t\lambda_g-2\alpha^2_tC^2_g)\ltwo{\theta_{t+1}-\theta^*}^2 -  2\alpha_t\langle \rho(x_t)g(x_t,\theta_{t}) - g(\theta_t), \theta_{t} - \theta^* \rangle\nonumber\\
	&\quad\quad -2 \alpha_t D_g \lone{\hat{\rho}(x_t,w^*_\rho) - \rho(x_t)}  \ltwo{\theta_{t+1} - \theta^*}\nonumber\\
	&\quad\quad -2 \alpha_t D_g \lone{\hat{\rho}(x_t,w_{\rho,t})
		-\hat{\rho}(x_t,w^*_\rho)}   \ltwo{\theta_{t+1} - \theta}\nonumber\\
	&\quad\geq (1+2\alpha_t\lambda_g-2\alpha^2_tC^2_g)\ltwo{\theta_{t+1}-\theta^*}^2 -  2\alpha_t\langle \rho(x_t)g(x_t,\theta_{t}) - g(\theta_t), \theta_{t} - \theta^* \rangle\nonumber\\
	&\quad\quad -\frac{1}{2}\alpha_t\lambda_g  \ltwo{\theta_{t+1} - \theta^*}^2 - \frac{2\alpha_tD^2_g}{\lambda_g}\lone{\hat{\rho}(x_t,w^*_\rho) - \rho(x_t)}^2\nonumber\\
	&\quad\quad -\frac{1}{2}\alpha_t\lambda_g  \ltwo{\theta_{t+1} - \theta^*}^2  - \frac{2\alpha_tD^2_g}{\lambda_g}\lone{\hat{\rho}(x_t,w_{\rho,t})
		-\hat{\rho}(x_t,w^*_\rho)}^2\nonumber\\
	&\quad= (1 + \alpha_t\lambda_g-2\alpha^2_tC^2_g)\ltwo{\theta_{t+1}-\theta^*}^2 -  2\alpha_t\langle \rho(x_t)g(x_t,\theta_{t}) - g(\theta_t), \theta_{t} - \theta^* \rangle\nonumber\\
	&\quad\quad - \frac{2\alpha_tD^2_g}{\lambda_g}\lone{\hat{\rho}(x_t,w^*_\rho) - \rho(x_t)}^2 - \frac{2\alpha_tD^2_g}{\lambda_g}\lone{\hat{\rho}(x_t,w_{\rho,t})
		-\hat{\rho}(x_t,w^*_\rho)}^2,\label{eq: 15}
\end{flalign}
where the first inequality follows from \Cref{lemma6}, and the third inequality follows from Young's inequality. Taking expectation on both sides of \cref{eq: 15} yields
\begin{flalign}
	(1 &+ \alpha_t\lambda_g-2\alpha^2_tC^2_g)\mE[\ltwo{\theta_{t+1}-\theta^*}^2]\nonumber\\
	&\leq \mE[\ltwo{\theta_t-\theta^*}^2] + 2\alpha^2_t\mE[\ltwo{\rho(x_t)g(x_t,\theta_{t}) - g(\theta_t)}^2] + \frac{2\alpha_tD^2_g}{\lambda_g}\mE\left[\lone{\hat{\rho}(x_t,w^*_\rho) - \rho(x_t)}^2\right] \nonumber\\
	&\quad + \frac{2\alpha_tD^2_g C^2_\psi}{\lambda_g}\mE\left[\lone{\hat{\rho}(x_t,w_{\rho,t})
		-\hat{\rho}(x_t,w^*_\rho)}^2\right] \nonumber\\
	&\leq \mE[\ltwo{\theta_t-\theta^*}^2] + 2V_g\alpha^2_t + \frac{2\alpha_tD^2_g C^2_\psi}{\lambda_g}\mE\left[\ltwo{w_{\rho,t} - w^*_\rho}^2\right] + \frac{2D^2_g \alpha_t\varepsilon_\rho}{\lambda_g},\label{eq: 18}
\end{flalign}
where the last inequality follows from \Cref{lemma8}. 

Substituting \cref{eq: 17} into \cref{eq: 18} yields
\begin{flalign}
	(1 + \alpha_t\lambda_g-&2\alpha^2_tC^2_g)\mE[\ltwo{\theta_{t+1}-\theta^*}^2]\nonumber\\
	&\leq \mE[\ltwo{\theta_t-\theta^*}^2] + 2V_g\alpha^2_t + \frac{2\alpha_tD^2_g D_\rho C^2_\psi}{\lambda_g(t+t_0)} + \frac{2D^2_g \alpha_t \varepsilon_\rho}{\lambda_g}.\label{eq: 19}
\end{flalign}

Multiplying both sides of \cref{eq: 19} with $r_t$ and summing over $t=0,\cdots,T-1$ yield
\begin{flalign}
&\sum_{t=0}^{T-1}a^\prime_t\mE[\ltwo{\theta_{t+1}-\theta^*}^2]\nonumber\\
&\quad\leq \sum_{t=0}^{T-1} r_t\mE[\ltwo{\theta_t-\theta^*}^2] + 2V_g\sum_{t=0}^{T-1} r_t\alpha^2_t + \frac{2D^2_g D_\rho C^2_\psi}{\lambda_g(t+t_0)} \sum_{t=0}^{T-1} r_t\alpha_t + \frac{2D^2_g \varepsilon_\rho}{\lambda_g}\sum_{t=0}^{T-1} r_t\alpha_t,\label{eq: 20}
\end{flalign}
where
\begin{flalign*}
	a^\prime_t = (1 + \alpha_t\lambda_g-2\alpha^2_tC^2_g)r_t.
\end{flalign*}
Now we let
\begin{flalign*}
	r_t &= (t+t_1)(t+t_1+1),\nonumber\\
	\alpha_t &= \frac{4}{\lambda_g(t+t_1-1)},\nonumber\\
	t_1&=\frac{16C^2_g}{\lambda^2_g} + 1.\nonumber
\end{flalign*}
We can obtain the following
\begin{flalign*}
	a^\prime_t - r_{t+1}  &= (1 + \alpha_t\lambda_g-2\alpha^2_tC^2_g)r_t - r_{t+1}\nonumber\\
	&\geq (1 + \alpha_t\lambda_g)r_t - (1 + 2\alpha^2_tC^2_g) r_{t+1}\nonumber\\
	&\geq (1 + \alpha_t\lambda_g)r_t - \left(1 + \frac{1}{2}\alpha_t\lambda_g\right) r_{t+1}\nonumber\\
	&\geq (t+t_1+1) \frac{(t+t_1)(t+t_1+3)-(t+t_1+2)^2}{t+t_1+1}\nonumber\\
	&\geq 0,
\end{flalign*}
where the second inequality follows from the fact that $\alpha_t\leq \frac{\lambda_g}{4C^2_g}$.

Substituting the above result to \cref{eq: 20} yields
\begin{flalign*}
	a^\prime_{T-1}\mE[\ltwo{\theta_{T}-\theta^*}^2]\leq r_0\ltwo{\theta_0-\theta^*}^2 + 2V_g\sum_{t=0}^{T-1} r_t\alpha^2_t + \frac{2D^2_g D_\rho C^2_\psi}{\lambda_g(t+t_0)} \sum_{t=0}^{T-1} r_t\alpha_t + \frac{2D^2_g \varepsilon_\rho}{\lambda_g}\sum_{t=0}^{T-1} r_t\alpha_t.
\end{flalign*}
The above inequality implies the following convergence rate
\begin{flalign*}
	\mE[\ltwo{\theta_{T}-\theta^*}^2]\leq \frac{r_0\ltwo{\theta_0-\theta^*}^2}{(T+t_1-1))(T+t_1)} + \frac{128V_g}{\lambda^2_g(T+t_1)} + \frac{64D^2_gC^2_gC^2_\psi\lambda^2_P}{9C^2_P\lambda^3_g(T+t_1)} + \frac{16D^2_g}{\lambda^2_g}\varepsilon_\rho,
\end{flalign*}
which completes the proof.

\section{Proof of \Cref{thm3}}

Following the similar argument similar to that in \cite[Lemma 4.2]{cai2019neural} and \cite[Theorem 1]{tsitsiklis1997analysis}, we can prove that ${\rm\Phi} \theta^*$ is the fixed point of the composite operator $\rpi_{{\rm\Phi},\mu_\pi}\bar{\mct}_\pi$. We then proceed as follows
\begin{flalign}
\lalpha{{\rm\Phi} \theta^* - G_\pi } & = \lalpha{{\rm\Gamma}_{{\rm\Phi},\mu_\pi} \bar{\mct}_\pi{\rm\Phi} \theta^* - {\rm\Gamma}_{{\rm\Phi},\mu_\pi}G_\pi + {\rm\Gamma}_{{\rm\Phi},\mu_\pi}G_\pi - G_\pi }\nonumber\\
&\leq \lalpha{{\rm\Gamma}_{{\rm\Phi},\mu_\pi} \bar{\mct}_\pi {\rm\Phi} \theta^* - {\rm\Gamma}_{{\rm\Phi},\mu_\pi}G_\pi } + \lalpha{ {\rm\Gamma}_{{\rm\Phi},\mu_\pi}G_\pi - G_\pi }\nonumber\\
&= \lalpha{{\rm\Gamma}_{{\rm\Phi},\mu_\pi} \bar{\mct}_\pi{\rm\Phi} \theta^* - {\rm\Gamma}_{{\rm\Phi},\mu_\pi}\bar{\mct}_\pi G_\pi } + \lalpha{ {\rm\Gamma}_{{\rm\Phi},\mu_\pi}G_\pi - G_\pi }\nonumber\\
&\leq \lalpha{ {\rm\Gamma}_{{\rm\Phi},\mu_\pi} [ \bar{\mct}_\pi {\rm\Phi} \theta^* - \bar{\mct}_\pi G_\pi ] } + \lalpha{ {\rm\Gamma}_{{\rm\Phi},\mu_\pi}G_\pi - G_\pi }\nonumber\\
&\leq \lalpha{ \bar{\mct}_\pi {\rm\Phi} \theta^* - \bar{\mct}_\pi G_\pi  } + \lalpha{ {\rm\Gamma}_{{\rm\Phi},\mu_\pi}G_\pi - G_\pi }\nonumber\\
&\leq  \gamma_G\lalpha{{\rm\Phi} \theta^* - G_\pi  } + \lalpha{ {\rm\Gamma}_{{\rm\Phi},\mu_\pi}G_\pi - G_\pi},\label{eq: 28}
\end{flalign}
where the first equality follows from the fact that ${\rm\Gamma}_{{\rm\Phi},\mu_\pi} \bar{\mct}_\pi {\rm\Phi} \theta^* = {\rm\Phi} \theta^*$, the second equality follows from the fact that $\bar{\mct}_\pi G_\pi = G_\pi$, the third inequality follows from the non-expansive property of the projection operator ${\rm\Gamma}_{{\rm\Phi},\mu_\pi}$, and the last inequality follows from \Cref{cond1}. \Cref{eq: 28} implies the following result
\begin{flalign*}
\lalpha{{\rm\Phi} \theta^* - G_\pi } \leq \frac{1}{1-\gamma_G} \lalpha{ {\rm\Gamma}_{{\rm\Phi},\mu_\pi}G_\pi - G_\pi}.
\end{flalign*}

\section{Extension to Case $\gamma_{\max}=1$}\label{sc: extension}
As shown in \Cref{cond1}, the operator $\bar{\mct}_{G,\pi}$ is not necessarily a contraction when $\gamma_{\max}=1$. The uniqueness of $G_\pi$ and $\hat{G}_\pi$ is not guaranteed in this case. We next consider the following assumption for the base matrix ${\rm\Phi}_i$, which can yield a desired property as we show below. Such an assumption has also been considered in the average reward MDP setting \cite{tsitsiklis1997average}.
\begin{assumption}[Non-constant Parameterization]\label{ass2}
	For all $i=1,\cdots,k$, we have ${\rm\Phi}_i\theta_i\neq c\mathbf{1}$ for any $\theta_i\in\mR^{d_i}$ and $c\in\mR/0$.
\end{assumption}
Despite the non-contraction nature of $\bar{\mct}_{G,\pi}$, if the base function ${\rm\Phi}_i$ satisfies \Cref{ass2}, we can show that the monotonicity condition of $g(\theta)$ in \Cref{cond2} still holds with a positive constant $\lambda_G$. As a result, the convergence bound in \cref{eq: 51} of \cref{thm1} is directly applicable to this setting with the corresponding value of $\lambda_G$. We can then further establish a result similar to \Cref{thm3} for the case with $\gamma_{\max}=1$ under \Cref{ass2}.

We first extend \Cref{cond2} and \Cref{thm3} to the case in which $\gamma_{\max}=1$. Without loss of generality, we consider $\gamma_i=1$ for all $i=1,\cdots,k$.

{\bf Forward GVF.} We first verify \Cref{cond2}. In this setting, we can still obtain the same result for $G$ as in \cref{eq: 61}, but with $A_i = [{\rm\Phi}_i^\top \bar{U}_{\pi} (\msP_\pi - I){\rm\Phi}_i] \otimes \msI_{d_i}$, where $\bar{U}_\pi = \text{diag}(\mu_\pi)$. As shown in \cite[Lemma 7]{tsitsiklis1997average}, the matrix $[{\rm\Phi}_i^\top U_{\pi,i} (\msP_\pi - I){\rm\Phi}_i]$ is Hurwitz when the base matrix ${\rm\Phi}_i$ satisfies \Cref{ass2}. Following the steps similar to those in \cref{eq: 63} - (\ref{eq: 66}), we can conclude that the matrix $G$ is also Hurwitz, which completes the proof.

We then verify \Cref{thm3}. We proceed as follows,
\begin{flalign}
	&\lalpha{\mcphi\theta^* - G_\pi}\nonumber\\
	&\quad=\lalpha{\rpi_{{\rm\Phi},\mu_\pi}{\mct}_\pi \mcphi\theta^* - {\mct}_\pi G_\pi} \nonumber\\
	&\quad\leq \lalpha{\rpi_{{\rm\Phi},\mu_\pi}{\mct}_\pi \mcphi\theta^* - \rpi_{{\rm\Phi},\mu_\pi}{\mct}_\pi G_\pi} + \lmu{\rpi_{{\rm\Phi},\mu_\pi} \mct_\pi G_\pi - {\mct}_\pi G_\pi}\nonumber\\
	&\quad\leq \lmu{\rpi_{{\rm\Phi},\mu_\pi}M_\pi(\mcphi\theta^* - G_\pi)} + \lalpha{\rpi_{{\rm\Phi},\mu_\pi}\mct_\pi G_\pi - \mct_\pi G_\pi}\nonumber\\
	&\quad\leq \lmu{M_\pi(\mcphi\theta^* - G_\pi)} + \lalpha{\rpi_{{\rm\Phi},\mu_\pi}\mct_\pi G_\pi - \mct_\pi G_\pi}\nonumber\\
	&\quad\leq C_\zeta \lmu{ \mcphi\theta^* - G_\pi} + \lalpha{\rpi_{{\rm\Phi},\mu_\pi}\mct_\pi G_\pi - \mct_\pi G_\pi},\label{eq: 67}
\end{flalign}
where the last inequality in \cref{eq: 67} can be obtained as follows. Following the steps similar to those in \cref{eq: 63}-(\ref{eq: 66}), we can conclude that $\text{eig}(M_\pi)=\text{eig}(\msP_\pi)$. For an ergodic MDP, we have $\max[\text{eig}(\msP_\pi)]=\max[\text{eig}(\msP_\pi^\top)] = 1$. Let $i = \argmax_{j}[\text{eig}(\msP_\pi)_j]$. We then have $\max_{j\neq i}\text{eig}(\msP_\pi)_j < 1$. Let $G_\pi$ be the fixed point of $\mct_\pi$ that is perpendicular to $[c_1\mathbf{1}_{d_1},\cdots,c_k\mathbf{1}_{d_k}]$, where $c_1,\cdots,c_k$ could be any constant. The vector $\mcphi\theta^*-G_\pi$ is perpendicular to the space spanned by the eigenvectors of $M_\pi$ associated with the eigenvalue $1$. Thus, there exists a positive constant $C_\zeta<1$ such that $\lmu{M_\pi(\mcphi\theta^* - G_\pi)}\leq C_\zeta \lmu{ \mcphi\theta^* - G_\pi}$, which yields the following results
\begin{flalign}
    \lalpha{\mcphi\theta^* - G_\pi} \leq \frac{1}{1-C_\zeta}\lalpha{\rpi_{{\rm\Phi},\mu_\pi}\mct_\pi G_\pi - \mct_\pi G_\pi}.\label{eq: 74}
\end{flalign}

{\bf Bakcward GVF.} To verify \Cref{cond2}, we can obtain the same result for $G$ as in \cref{eq: 61} with $A_i = [{\rm\Phi}_i^\top (\msP^\top_\pi - I) \bar{U}_{\pi} {\rm\Phi}_i] \otimes \msI_{d_i}$. Define $\bar{A}_i={\rm\Phi}_i^\top (\msP^\top_\pi - I) U_{\pi} {\rm\Phi}_i$. We next show that $\bar{A}_i$ is Hurwitz. Note that $\bar{A}_i = \mE_{\mu_\pi}[\phi^\prime(\phi - \phi^\prime)]$. Let $z$ be a non-constant function on the state-action space. Then we have
\begin{flalign}
	0&< \frac{1}{2}\mE_{\mu_\pi}[(z(s,a)-z(s^\prime,a^\prime))^2]\nonumber\\
	& = \mE_{\mu_{\pi}}[z(s,a)^2] - \mE[z(s,a)z(s^\prime,a^\prime)]\nonumber\\
	& = z^\top \bar{U}_{\pi} z - z^\top \msP_\pi \bar{U}_\pi z\nonumber\\
	& = z^\top(I-\msP_\pi)\bar{U}_\pi z.\label{eq: 38}
\end{flalign}
For a vector $v\in \mR^{K_i}$, we have
\begin{flalign}
	v^\top \bar{A}_i v = v{\rm\Phi}_i^\top(\msP_\pi^\top - I)U_\pi {\rm\Phi}_iv\label{eq: 72}.
\end{flalign}
Since ${\rm\Phi}v$ is a non-constant function, \cref{eq: 38} and \cref{eq: 72} together imply that
\begin{flalign*}
	v^\top \bar{A}_i v < 0 \quad \text{for all}\quad v\in \mR^{K_i}.
\end{flalign*}
Thus, the matrix $\bar{A}_i$ is Hurwitz, which further implies that $A_i$ is also Hurwitz. Following the steps similar to those in \cref{eq: 63} - (\ref{eq: 66}), we can conclude that the matrix $G$ is also Hurwitz, which completes the proof.

We then verify \Cref{thm3}. We proceed as follows,
\begin{flalign}
	&\lalpha{\mcphi\theta^* - G_\pi}\nonumber\\
	&\quad =\lalpha{\rpi_{{\rm\Phi},\mu_\pi}{\mct}_\pi \mcphi\theta^* - {\mct}_\pi G_\pi} \nonumber\\
	&\quad\leq \lalpha{\rpi_{{\rm\Phi},\mu_\pi}{\mct}_\pi \mcphi\theta^* - \rpi_{{\rm\Phi},\mu_\pi}{\mct}_\pi G_\pi} + \lmu{\rpi_{{\rm\Phi},\mu_\pi} \mct_\pi G_\pi - {\mct}_\pi G_\pi}\nonumber\\
	&\quad\leq \lmu{\rpi_{{\rm\Phi},\mu_\pi}\hat{M}_\pi(\mcphi\theta^* - G_\pi)} + \lalpha{\rpi_{{\rm\Phi},\mu_\pi}\mct_\pi G_\pi - \mct_\pi G_\pi}\nonumber\\
	&\quad\leq \lmu{\hat{M}_\pi(\mcphi\theta^* - G_\pi)} + \lalpha{\rpi_{{\rm\Phi},\mu_\pi}\mct_\pi G_\pi - \mct_\pi G_\pi}\nonumber\\
	&\quad\leq C_\zeta \lmu{ \mcphi\theta^* - G_\pi} + \lalpha{\rpi_{{\rm\Phi},\mu_\pi}\mct_\pi G_\pi - \mct_\pi G_\pi},\label{eq: 73}
\end{flalign}
where the last inequality in \cref{eq: 73} can be obtained as follows. Using \cite[Theorem 1.3.22]{horn2012matrix}, we have
\begin{flalign*}
	\text{eig}(U^{-1}_{\pi,i}[\msP_{\pi,i}^\top \otimes \msI_{d_i}] U_{\pi,i}) = \text{eig}([\msP_{\pi,i}^\top \otimes \msI_{d_i}] U_{\pi,i} U^{-1}_{\pi,i}) = \text{eig}([\msP_\pi^\top \otimes \msI_{d_i}]) = \text{eig}(\msP_\pi^\top)=\text{eig}(\msP_\pi).
\end{flalign*}
Following the steps similar to those in \cref{eq: 63}-(\ref{eq: 66}), we can conclude that $\text{eig}(\hat{M}_\pi)=\text{eig}(\msP_\pi)$. Following the steps  similar to those for obtaining \cref{eq: 74}. We have
\begin{flalign*}
    \lalpha{\mcphi\theta^* - G_\pi} \leq \frac{1}{1-C_\zeta}\lalpha{\rpi_{{\rm\Phi},\mu_\pi}\mct_\pi G_\pi - \mct_\pi G_\pi},
\end{flalign*}
where $0<C_\zeta<1$, which completes the proof.

\section{Proof of \Cref{thm2}}\label{sc: GTD_proof}
We first define the matrix $B$ in the following way: 
\begin{itemize}
    \item Forward GVF: $B=\mE_{\mcd\cdot \pi}[[\phi(s^\prime,a^\prime)\otimes \msI_d] m(x) [\phi(s,a)\otimes\msI_d]]$
    \item Backward GVF: $B=\mE_{\mcd\cdot \pi}[[\phi(s,a)\otimes \msI_d] m(x) [\phi(s^\prime,a^\prime)\otimes\msI_d]]$.
\end{itemize}
We further define the following stochastic matrices in both the forward and backward GVF evaluation settings. Recall that $(s_t,a_t)\sim D(\cdot)$, $s^\prime_t\sim\msP(\cdot|s_t,a_t)$ and $a^\prime_t\sim \pi(\cdot|s^\prime_t$.
\begin{itemize}
    \item Forward GVF:
    \begin{flalign}
	A_t &= [\phi(s_t,a_t)\otimes\msI_d](m(x_t)[\phi(s^\prime_t,a^\prime_t)\otimes\msI_d]^\top - [\phi(s_t,a_t)\otimes\msI_d]^\top),\nonumber\\
	B_t &= [\phi(s^\prime_t,a^\prime_t)\otimes \msI_d] m(x_t) [\phi(s_t,a_t)\otimes\msI_d],\nonumber\\
	C_t &=(\phi(s_t,a_t)\phi(s_t,a_t)^\top)\otimes\msI_d,\nonumber\\
	b_t &= [\phi(s_t,a_t)\otimes \msI_d]C(x_t).
\end{flalign}
    \item Backward GVF:
    \begin{flalign}
    A_t &= [\phi(s^\prime_t,a^\prime_t)\otimes\msI_d](m(x_t)[\phi(s_t,a_t)\otimes\msI_d]^\top - [\phi(s^\prime_t,a^\prime_t)\otimes\msI_d]^\top),\nonumber\\
    B_t &= [\phi(s_t,a_t)\otimes \msI_d] m(x_t) [\phi(s^\prime_t,a^\prime_t)\otimes\msI_d],\nonumber\\
    C_t &=(\phi(s^\prime_t,a^\prime_t)\phi(s_t,a_t)^\top)\otimes\msI_d,\nonumber\\
    b_t &= [\phi(s^\prime_t,a^\prime_t)\otimes \msI_d]C(x_t).
\end{flalign}
\end{itemize}

Recall the matrices $A$ and $C$ defined in \Cref{sc: GTD}. For a constant $\xi>0$, we define
\begin{flalign}
H_t &= \left[\begin{array}{cc}
A_t & B_t \\
\xi A_t & \xi C_t 
\end{array}
\right],\qquad
h_t = \left[\begin{array}{c}
b_t\\
0
\end{array}
\right].\nonumber
\end{flalign}
and 
\begin{flalign}
H &= \left[\begin{array}{cc}
A & B \\
\xi A & \xi C 
\end{array}
\right],\qquad
h = \left[\begin{array}{c}
b\\
0
\end{array}
\right],\nonumber
\end{flalign}
where $A=\mE[A_t]$, $B=\mE[B_t]$, $C=\mE[C_t]$ and $b = \mE[b_t]$.

For the matrix $H_t$, we have the following holds
\begin{flalign}
\lF{H_t}^2&= (1+\xi^2)\lF{A_t}^2 + \lF{B_t}^2 + \xi^2\lF{C_t}^2\nonumber\\
&\leq (1+\xi^2) [d^2C^2_\phi (C_m+1)]^2 + d^2C^2_\phi C^2_m + \xi^2 C^4_\phi d^2.\label{eq: 71.5}
\end{flalign}
which implies that $\lF{H_t}\leq C_H$, where
\begin{flalign*}
C_H=\sqrt{(1+\xi^2) [d^2C^2_\phi (C_m+1)]^2 + d^2C^2_\phi C^2_m + \xi^2 C^4_\phi d^2}.
\end{flalign*}
For the vector $h_t$, we can obtain $\ltwo{h_t}\leq C_h = dC_\phi R_C$ by following the steps similar to those for obtaining \cref{eq: 71.5}.

The update in \Cref{algorithm_gtd} can be rewritten as
\begin{flalign}\label{eq: 22}
	v_{t+1} = {\rm\Gamma}_{R_v}\left(v_t + \alpha_t (H_tv_t + h_t)\right),
\end{flalign}
where $v_t=[\theta_t^\top,w_t^\top]^\top$, and $R_v = R_\theta\times \mR^{Kd_g\times 1}$. Following the proof similar to those in \cite[Section 5.3.3, Theorem 3]{maei2011gradient}, we can show that the matrix $H$ is Hurwitz under \Cref{ass4} and \Cref{ass5} with an appropriately chosen $\xi>\max\{0, -\text{eig}_{\min}(C^{-1}[(A+A^\top)/2])\}$.

We define the following optimal point $v^*=[\bar{\theta}^{*\top},w^{*\top}]^\top$ for the linear SA defined in \cref{eq: 22}
\begin{flalign*}
\langle \varphi(v^*), v - v^* \rangle \leq 0,\quad \forall v\in R_v,
\end{flalign*}
where $\varphi(v) = Hv + b$. We also define $C_v=\ltwo{v^*}$. It can be checked that there exist a positive constant $\lambda^\prime_G$ such that
\begin{flalign}
\langle  \varphi(v^*)-\varphi(v), v^*-v\rangle\leq -\lambda^\prime_G\ltwo{v-v^*}^2.\label{eq: cond_gtd}
\end{flalign}
We further define $\varphi(x_t,v)=H_tv + h_t$.

Following the steps similar to those for proving \Cref{lemma4} and \Cref{lemma5}, we can obtain the following two lemmas.
\begin{lemma}\label{lemma12}
	Given a sample $(s_t,a_t,B_t,s^\prime_t)\sim \mcd_d$ and $a^\prime_t\sim \pi(\cdot|s^\prime_t)$ and any $v\in R_v$, we have the following holds
	\begin{flalign*}
	\ltwo{\varphi(x_t,v) - \varphi(v)}^2\leq 16C^2_H\ltwo{\kappa-\kappa^*}^2 + 16C^2hPC^2_v + 8C^2_h.
	\end{flalign*}
\end{lemma}
\begin{proof}
	Based on the definition of $\varphi(x_t,v)$ and $\varphi(v)$, we can obtain the following
	\begin{flalign*}
	\ltwo{\varphi(x_t,v) - \varphi(v)}^2 & \leq 2\ltwo{(H_t-H)v}^2 + 2\ltwo{h_t - h}^2 \nonumber\\
	&\leq  4\ltwo{(H_t-H)(v-v^*)}^2 + 4\ltwo{(H_t-H)v^*}^2 + 2\ltwo{h_t - h}^2 \nonumber\\
	&\leq 16C^2_H\ltwo{\kappa-\kappa^*}^2 + 16C^2_hC^2_v + 8C^2_h.
	\end{flalign*}
\end{proof}
\begin{lemma}\label{lemma11}
	Consider the population GTD update $\varphi(v) = Hv + b$. We have
	\begin{flalign*}
	\langle -\varphi(v), v - v^* \rangle \geq \lambda^\prime_G\ltwo{v-v^*}^2,\quad \forall v\in R_v.
	\end{flalign*}
\end{lemma}
We also have the following "three-point lemma" holds for the GTD update.
\begin{lemma}\label{lemma10}
	Consider the update of $w_{t}$ and $\theta_{t}$ in \Cref{algorithm_gtd}. For all $v\in R_v$, we have the following holds
	\begin{flalign}
	-\alpha_t\langle \varphi(x_t,v_t), v_{t+1} - v\rangle + \frac{1}{2}\ltwo{v_{t+1}-v_t}^2\leq \frac{1}{2}\ltwo{v_t-v}^2 - \frac{1}{2}\ltwo{v_{t+1}-v}^2.\label{eq: 23}
	\end{flalign}
\end{lemma}
Using \Cref{lemma10} and following the steps similar to those from \cref{eq: 2} to \cref{eq: 5}, we can obtain
\begin{flalign}
&\frac{1}{2}\ltwo{v_t-v}^2 - \frac{1}{2}\ltwo{v_{t+1}-v}^2 \nonumber\\
&\quad\geq -\alpha_t \langle \varphi(v_{t+1}), v_{t+1} - v\rangle - \alpha_t \langle \varphi(x_t,v_t)-\varphi(v_{t}), v_{t} - v\rangle - \alpha^2_t C^2_H  \ltwo{v_{t+1} - v}^2 \nonumber\\
&\quad\quad - \alpha^2_t \ltwo{\varphi(x_t,\kappa_t)-\varphi(\kappa_{t})}^2.\label{eq: 24}
\end{flalign}
Taking expectation on both sides of \cref{eq: 24}, letting $v=v^*$, and using the fact that $-\langle \varphi(v_{t+1}), v_{t+1} - v^*\rangle\leq \lambda^\prime_G\ltwo{v_{t+1} - v^*}$ yield
\begin{flalign}
(1+2\alpha_t\lambda^\prime_G - 2\alpha^2_t C^2_H )\mE[\ltwo{v_{t+1}-v^*}^2]&\leq \mE[\ltwo{v_t-v^*}^2] + 2\alpha^2_t \mE[\ltwo{\varphi(x_t,\kappa_t)-\varphi(\kappa_{t})}^2]\nonumber\\
&\leq(1+32\alpha^2_tC^2_H)\mE[\ltwo{v_t-v^*}^2] + + 32C^2_hC^2_v + 16C^2_h,\label{eq: 25}
\end{flalign}
where the second inequality follows from \Cref{lemma11}. Multiplying both sides of \cref{eq: 25} by $o_t$ and summing over iterations $t=0,\cdots,T-1$ yield
\begin{flalign}
	\sum_{t=0}^{T-1}a^{\prime\prime}_t\mE[\ltwo{v_{t+1}-v^*}^2]\leq \sum_{t=0}^{T-1}b^{\prime\prime}_t\mE[\ltwo{v_{t}-v^*}^2] + c^{\prime\prime},\label{eq: 26}
\end{flalign}
where 
\begin{flalign}
	a^{\prime\prime}_t &= (1+2\alpha_t\lambda^\prime_G - 2\alpha^2_t C^2_H )o_t,\nonumber\\
	b^{\prime\prime}_t &= (1+32\alpha^2_tC^2_H)o_t,\nonumber\\
	c^{\prime\prime} &= (32C^2hPC^2_v + 16C^2_h)\sum_{t=0}^{T-1}\alpha^2_t o_t.\nonumber
\end{flalign}
Now we let
\begin{flalign*}
o_t &= (t+t_2)(t+t_2+1),\nonumber\\
\alpha_t &= \frac{4}{\lambda^\prime_G(t+t_2-1)},\nonumber\\
t_2&=\frac{34C^2_H}{\lambda^2_J} + 1.\nonumber
\end{flalign*}
Then, we can obtain the following
\begin{flalign*}
a^{\prime\prime}_t - b^{\prime\prime}_{t+1} &= (1 + 2\alpha_t\lambda^\prime_G -2\alpha^2_t C^2_H)s_t - (1+32\alpha_{t+1}^2C^2_H)o_{t+1}\nonumber\\
&\geq (1 + 2\alpha_t\lambda^\prime_G)o_t - (1+ 2\alpha^2_tC^2_H +32\alpha_{t+1}^2C^2_H)o_{t+1}\nonumber\\
&\geq (1 + 2\alpha_t\lambda^\prime_G)o_t - (1+ 34\alpha_{t}^2C^2_H)o_{t+1}\nonumber\\
&\geq (1 + 2\alpha_t\lambda^\prime_G)o_t - (1 + \alpha_t\lambda^\prime_G)o_{t+1}\nonumber\\
&\geq (t+t_2+1) \frac{(t+t_2+3)(t+t_2) - (t+t_2+2)^2}{t+t_2-1} \nonumber\\
&\geq 0,
\end{flalign*}
where the second inequality follows from the fact that $\alpha_t\leq \frac{\lambda^\prime_G}{34C^2_H}$.

Applying the above property to \cref{eq: 26} yields
\begin{flalign*}
a^{\prime\prime}_{T-1}t\mE[\ltwo{v_T-v^*}^2]\leq b^{\prime\prime}_0\ltwo{v_0-v^*}^2 + c^{\prime\prime},
\end{flalign*}
which implies
\begin{flalign*}
\mE[\ltwo{v_T-v^*}^2] &\leq \frac{b^{\prime\prime}_0\ltwo{v_0-v^*}^2}{a^{\prime\prime}_{T-1}} + \frac{c^{\prime\prime}}{a^{\prime\prime}_{T-1}}\nonumber\\
&\leq \frac{(1+16\alpha^2_0C^2_H)(t_2+1)^2\ltwo{v_0-v^*}^2}{(T+t_2-1)(T+t_2)} + \frac{128C^2_hC^2_v + 64C^2_h}{(T+t_2)\lambda^2_J}.
\end{flalign*}
Using the fact $\lF{\theta_T-\bar{\theta}^*}^2\leq \ltwo{v_T-v^*}^2$, we have
\begin{flalign*}
	\mE[\lF{\theta_T-\bar{\theta}^*}^2]\leq \frac{(1+16\alpha^2_0C^2_H)(t_2+1)^2\ltwo{v_0-v^*}^2}{(T+t_2-1)(T+t_2)} + \frac{128C^2_hC^2_v + 64C^2_h}{(T+t_2)\lambda^2_J},
\end{flalign*}
which completes the proof.

\end{document}